\documentclass{article}

\PassOptionsToPackage{numbers, compress}{natbib}

\usepackage[final]{neurips_2023}

\usepackage[utf8]{inputenc} 
\usepackage[T1]{fontenc}    
\usepackage{hyperref}       
\usepackage{url}            
\usepackage{booktabs}       
\usepackage{amsfonts}       
\usepackage{nicefrac}       
\usepackage{microtype}      
\usepackage{xcolor}         
\usepackage{enumitem}
\usepackage{amsmath}
\usepackage{amssymb}
\usepackage{mathtools}
\usepackage{amsthm}
\usepackage{ulem}[normalem]

\usepackage[capitalize,noabbrev]{cleveref}

\usepackage{bbold}
\usepackage{comment}
\newcommand\eq[1]{\begin{equation}#1\end{equation}}
\newcommand\mId[1]{\mathrm{Id}_{#1}}
\newcommand{\RR}{\mathbb{R}}
\newcommand{\vf}{\mathcal{X}(\Omega)}
\newcommand{\RD}{\mathbb{R}^D} 
\newcommand{\Rd}{\mathbb{R}^d} 
\newcommand{\ouv}{\Omega} 
\newcommand{\vdim}{\mathrm{dim}}
\newcommand{\linspan}{\mathrm{span}}
\newcommand{\range}{\mathrm{range}}
\newcommand{\lie}{\mathrm{Lie}}
\newcommand{\R}{\mathbb{R}}
\newcommand{\W}{W_{\phi}}

\newcommand{\rvline}{\hspace*{-\arraycolsep}\vline\hspace*{-\arraycolsep}}
\newcommand{\x}{\theta} 

\renewcommand{\epsilon}{\varepsilon}

\theoremstyle{plain}
\newtheorem{theorem}{Theorem}[section]
\newtheorem{proposition}[theorem]{Proposition}
\newtheorem{lemma}[theorem]{Lemma}
\newtheorem{corollary}[theorem]{Corollary}
\theoremstyle{definition}
\newtheorem{definition}[theorem]{Definition}
\newtheorem{assumption}[theorem]{Assumption}
\theoremstyle{remark}
\newtheorem{remark}[theorem]{Remark}
\newtheorem{example}[theorem]{Example}

\title{Abide by the Law and Follow the Flow: \protect\\
Conservation Laws for Gradient Flows}

\author{%
  Sibylle Marcotte
  \\
  ENS - PSL Univ.\\
  \texttt{sibylle.marcotte@ens.fr} \\
  \And
   Rémi Gribonval \\
  Univ Lyon, EnsL, UCBL, \\
  CNRS, Inria,  LIP, \\
  \texttt{remi.gribonval@inria.fr} \\
   \And
  Gabriel Peyré \\
  CNRS, ENS - PSL Univ. \\
  \texttt{gabriel.peyre@ens.fr} \\
}

\usepackage[textsize=tiny]{todonotes}

\begin{document}

\maketitle
\normalem 

\begin{abstract}
 Understanding the geometric properties of gradient descent dynamics is a key ingredient in deciphering the recent success of very large machine learning models. A striking observation is that trained over-parameterized models retain some properties of the optimization initialization. 
This ``implicit bias'' is believed to be responsible for some favorable properties of the trained models and could explain their good generalization properties. The purpose of this article is threefold. 
First, we rigorously expose the definition and basic properties of ``conservation laws'', that define quantities conserved during gradient flows of a given model (e.g. of a ReLU network with a given architecture) with any training data and any loss.
Then we explain how to find the maximal number of independent conservation laws
by performing finite-dimensional algebraic manipulations on the Lie algebra generated by the Jacobian of the model.
Finally, we provide algorithms to: a) compute a family of polynomial laws; b) compute the maximal number of (not necessarily polynomial) independent conservation laws.
We provide showcase examples that we fully work out theoretically. Besides, applying the two algorithms confirms for a number of ReLU network architectures that all known laws are recovered by the algorithm, and that there are no other independent laws. Such computational tools pave the way to understanding desirable properties of optimization initialization in large machine learning models.
\end{abstract}

\vspace{-0.3cm}
\section{Introduction} \label{Intro}
\vspace{-0.1cm}
State-of-the-art approaches in machine learning rely on the conjunction of gradient-based optimization with vastly ``over-parameterized'' architectures. A large body of empirical \cite{zhang} and theoretical  \cite{belkin} works suggest that, despite the ability of these models to almost interpolate the input data, they are still able to generalize well.  
Analyzing the training dynamics of these models is thus crucial to gain a better understanding of this phenomenon. 
Of particular interest is to understand what properties of the initialization are preserved during the dynamics, which is often loosely referred to as being an ``implicit bias'' of the training algorithm. 
The goal of this article is to make this statement precise, by properly defining maximal sets of 
such
``conservation laws'', 
by linking these quantities to algebraic computations (namely a Lie algebra) associated with the model parameterization (in our framework, this parameterization is embodied by a re-parameterization mapping $\phi$), and finally by exhibiting algorithms to implement these computations in SageMath \cite{sagemath}.
\paragraph{Over-parameterized model}

Modern machine learning practitioners and researchers have found that over-parameterized neural networks (with more parameters than training data points), which are often trained until perfect interpolation, have impressive generalization properties \cite{zhang,belkin}. This performance seemingly contradicts classical learning theory \cite{MLtheory}, and a large part of the theoretical deep learning literature aims at explaining this puzzle. 
The choice of the optimization algorithm is crucial to the model generalization performance \cite{gunasekar, neyshaburthesis, jigradient}, thus inducing an \textit{implicit bias}.

\paragraph{Implicit bias} 

The terminology ``implicit bias'' 
informally refers
to properties of trained models which are induced by the optimization procedure, typically some form of regularization~\cite{neyshabur}.  
For gradient descent, in 
simple cases such as scalar linear neural networks or two-layer networks with a single neuron, it is actually possible to compute in closed form the implicit bias, which induces some approximate or exact sparsity regularization~\cite{gunasekar}.
Another interesting case is logistic classification on separable data, where the implicit bias selects the max-margin classifier both for linear models~\cite{soudry} and for two-layer neural networks in the mean-field limit~\cite{chizat}.
A key hypothesis to explicit the implicit bias is often that the Riemannian metric associated to the over-parameterization is either of Hessian type~\cite{gunasekar,azulay21}, or can be somehow converted to be of Hessian type~\cite{azulay21}, which is seemingly always a strong constraint.
For example, even for simple two-layer linear models (i.e., matrix factorization) with more than a single hidden neuron,
the Hessian type assumption does not hold, and no closed form is known for the implicit bias~\cite{gunasekar2017implicit}.
The work of~\cite{li22} gives conditions on the over-parameterization for this to be possible (for instance certain Lie brackets should vanish). These conditions are (as could be expected) stronger than those required to apply Frobenius theory, as we do in the present work to retrieve conservation laws.

\vspace{-0.1cm}
\paragraph{Conservation laws}
Finding functions conserved during gradient flow optimization of neural networks (a continuous limit of gradient descent often used to model the optimization dynamics) is particularly useful to better understand the flow behavior.
One can see conservation laws as a “weak” form of implicit bias: to explain, among a possibly infinite set of minimizers, which properties (e.g. in terms of sparsity, low-rank, etc.) are being favored by the dynamic. If there are enough conservation laws, one has an exact description of the dynamic (see \Cref{sec:riemanian-flow}), and in some cases, one can even determine explicitly the implicit bias. Otherwise, one can still predict what properties of the initialization are retained at convergence, and possibly leverage this knowledge.
For example, in the case of linear neural networks, certain \emph{balancedness properties} are satisfied and provide a class of conserved functions~\cite{Saxe,Du,Arora18a,Arora18b, Ji,Tarmoun, Min}.
These conservation laws enable for instance to 
prove the global convergence of the gradient flow~\cite{Bah} under some assumptions. 
We detail these 
laws in 
Proposition~\ref{conservation}.
A subset of these ``balancedness'' laws still holds in the case of a ReLU activation \cite{Du}, which reflects the rescaling invariance of these networks (see \Cref{sectionlinearNN} for more details).
More generally such conservation laws bear connections with the invariances of the model \cite{Kunin}: to each 1-parameter group of transformation preserving the loss, one can associate a conserved quantity, which is in some sense analogous to Noether's theorem \cite{Noether1918, Tanaka, głuch2021noether}. Similar reasoning is used by \cite{zhao} to show the influence of initialization on convergence and generalization performance of the neural network.
Our work is somehow complementary to this line of research: instead of assuming a priori known symmetries, we directly analyze the model and give access to conservation laws using algebraic computations.  
For matrix factorization as well as for certain ReLU network architectures, this allows us to show that the conservation laws reported in the literature are complete (there are no other independent quantities that would be preserved by all gradient flows).

\subsection*{Contributions}

We formalize the notion of a conservation law, a quantity preserved through all gradient flows given a model architecture (e.g. a ReLU neural network with prescribed layers). 
Our main contributions are:
\begin{itemize}[leftmargin=0.5cm]
    \item to show that for several classical losses, characterizing conservation laws for deep linear (resp. shallow ReLU)
    networks boils down to analyzing a finite dimensional space of vector fields; 
    \item to propose an algorithm (coded in SageMath) identifying polynomial conservation laws on linear / ReLU network architectures; it identifies all known laws on selected examples; 
    \item to formally define the maximum number of (not necessarily polynomial) independent conservation laws and characterize it a) theoretically via Lie algebra computations; and b) practically via an algorithm (coded in SageMath) computing this number on worked examples;
    \item to illustrate that in certain settings these findings allow to rewrite an over-parameterized flow as an ``intrinsic'' low-dimensional flow;
    \item to highlight that the cost function associated to the training of linear and ReLU networks, shallow or deep, with various losses (quadratic and more) fully fits the proposed framework.
\end{itemize}
A consequence of our results is to show for the first time that conservation laws commonly reported in the literature are maximal: there is no other independent preserved quantity (see Propositions~\ref{dimlie}, \ref{classicinv}, \Cref{nbatteint}) and \Cref{section:numerics}).

\section{Conservation Laws for Gradient Flows}
After some reminders on gradient flows, we formalize the notion of conservation laws.
\vspace{-0.1cm}
\subsection{Gradient dynamics} \label{Context}
\vspace{-0.2cm}
We consider learning problems, where we denote $x_i \in \RR^m$ the features and $y_i \in \mathcal{Y}$ the targets (for regression, typically with $\mathcal{Y} = \RR^n$) or labels (for classification) in the case of supervised learning, while $y_i$ can be considered constant for unsupervised/self-supervised learning. We denote $X \coloneqq (x_i)_i$ and $Y \coloneqq (y_i)_i$.
Prediction is performed by a parametric mapping $g(\theta, \cdot): \RR^m \to \RR^n$ (for instance a neural network) which is trained by 
empirically minimizing over parameters $\x \in \Theta \subseteq \RD$ a \textbf{cost} 
\begin{equation}\label{eq:erm}
        \mathcal{E}_{X,Y}(\theta) \coloneqq \sum_i \ell(g({\x},x_i),y_i),     
\end{equation} 
where $\ell$ is the \textbf{loss} function.
In practical examples with linear or ReLU networks, $\Theta$ is either 
$\RD$ or an open set of ``non-degenerate'' parameters.
The goal of this paper is to analyze what functions $h(\theta)$ are preserved during the gradient flow (the continuous time limit of gradient descent) of $\mathcal{E}_{X,Y}$:
%
\begin{equation} \label{gradientflow}
        \overset{.}{\x}(t) =  -\nabla \mathcal{E}_{X,Y} (\x(t)), \text{ with }
        \x(0) = \x_{\text{init}}.
\end{equation}
A priori, one can consider different ``levels'' of conservation, depending whether $h$ is conserved: during the optimization of $\mathcal{E}_{X,Y}$ for a given loss $\ell$ and a given data set $(x_i,y_i)_i$; 
or given a loss $\ell$, during the optimization of $\mathcal{E}_{X,Y}$ for {\em any} data set $(x_i,y_i)_i$. 
Note that using stochastic optimization methods and discrete gradients would break the exact preservation of the conservation laws, and only approximate conservation would hold, as remarked in~\cite{Kunin}. 
\vspace{-0.1cm}
\subsection{Conserved functions} 
\vspace{-0.2cm}
As they are based on gradient flows, conserved functions are first defined locally.
\begin{definition}[Conservation through a flow]  
\label{def:conserved_through_flow}
Consider an open subset $\ouv \subseteq \Theta$ and a vector field $\chi \in \mathcal{C}^1(\ouv, \RD)$. 
By the Cauchy-Lipschitz theorem, for each initial condition $\x_{\text{init}}$, there exists a unique maximal solution $t \in [0, T_{\x_{\text{init}}}) \mapsto \x(t,\x_{\text{init}})$ of the ODE $ \dot \x(t) = \chi(\x(t))$ with $\x(0) = \x_{\text{init}}$. 
A function $h: \ouv \subseteq \R^D \to \RR$ 
 is {\em conserved on $\ouv$ through the vector field $\chi$}  if $h(\x(t,\x_{\text{init}}))=h(\x_{\text{init}})$ for each choice of $\x_{\text{init}} \in \ouv$  and every $t \in [0, T_{\x_{\text{init}}})$.
 It is {\em conserved on $\ouv$ through a subset} $W \subset \mathcal{C}^1(\ouv, \RD)$ if $h$ is conserved on $\ouv$ during all flows induced by all $\chi \in W$. 
%

\end{definition}
In particular, one can adapt this definition for the flow induced by the cost  \eqref{eq:erm}.
\begin{definition}[Conservation during the flow \eqref{gradientflow} with a given dataset]
Consider an open subset $\ouv \subseteq \Theta$ and a dataset $(X, Y)$ such that $\mathcal{E}_{X, Y} \in \mathcal{C}^{2}(\ouv, \R)$. A function $h: \ouv \subseteq \R^D \to \RR$ is {\em conserved on $\ouv$ during the flow} \eqref{gradientflow} if it is conserved through the vector field $\chi(\cdot) \coloneqq \nabla \mathcal{E}_{X, Y} (\cdot)$.
\end{definition}

Our goal is to study which functions are conserved during {\em ``all'' flows} defined by the ODE \eqref{gradientflow}. This in turn leads to the following definition.
\begin{definition}[Conservation during the flow \eqref{gradientflow} with ``any'' dataset]
  Consider an open subset $\Omega \subset \Theta$ and a loss $\ell(z, y)$ such that $\ell(\cdot, y)$ is $\mathcal{C}^2$-differentiable for all $y \in \mathcal{Y}$. 
  A function $h: \ouv \subseteq \R^D \to \RR$ is {\em conserved on $\ouv$ for any data set} if, for each data set $(X, Y)$ {\em such that} $g(\cdot, x_i) \in \mathcal{C}^{2}(\ouv, \R^n)$ for each $i$, the function $h$ is conserved on $\ouv$ during the flow \eqref{gradientflow}. This leads us to introduce the family of vector fields:
\eq{ \label{eq:W_ell}
 W_\Omega^{g} := \left\{ \chi(\cdot): \exists X,Y, \forall i\ g(\cdot, x_i)
\in \mathcal{C}^{2}(\Omega,\R^n)
 ,\ \chi = \nabla \mathcal{E}_{X,Y} \right\}
  \subseteq \mathcal{C}^{1}(\ouv, \RD)
 }
so that being conserved on $\Omega$ for any dataset is the same as being conserved on $\Omega$ through $W_\Omega^{g}$.
\end{definition}
The above definitions are local and conditioned on a choice of open set of parameters $\ouv \subset \Theta$. We are rather interested in functions defined on the whole parameter space $\Theta$, hence the following definition.
\begin{definition} \label{def:locally_conserved_any}
    A function $h: \Theta \mapsto \R$ is {\em locally conserved on $\Theta$ for any data set} if for each open subset $\ouv \subseteq \Theta$, $h$ is conserved on $\ouv$ for any data set.   
\end{definition}
A basic property of $\mathcal{C}^1$ 
conserved functions (which proof can be found in \Cref{appendix-orth}) 
corresponds to an ``orthogonality'' between their gradient and the considered vector fields.
\begin{proposition}\label{tracecharact}
Given a subset $W \subset \mathcal{C}^{1}(\ouv, \RD)$, its \emph{trace} at $\x \in \ouv$ is defined as the linear space
\vspace{-0.1cm}
\begin{equation}\label{eq:DefTrace}
W(\x) \coloneqq \linspan{\{\chi(\x): \chi \in W\}} \subseteq \RD.
\end{equation}
A function $h \in \mathcal{C}^1(\ouv, \mathbb{R})$ is conserved on $\ouv$ through $W$
if, and only if $\nabla h(\x) \perp W(\x),\forall \x \in \ouv$.
\end{proposition}

Therefore, combining \Cref{tracecharact} and \Cref{def:locally_conserved_any}, the object of interest to study locally conserved functions is the union of the traces
\begin{equation} \label{eq:def_W_ell}
W^{g}_\x \coloneqq \bigcup \Big\{
    W_{\ouv}^{g}(\x) \;:\; 
    \ouv \subseteq \Theta \text{ with } \ouv 
    \text{ a neighborhood of } \x \Big\}.
\end{equation}
\begin{corollary} \label{coro:perp}
    A function $h: \Theta \mapsto \R$ is locally conserved on $\Theta$ for any data set if and only if $\nabla h(\x) \perp W^{g}_\x$ for all $\x \in \Theta$.
\end{corollary}
It will 
soon be shown (cf \Cref{theorem:reformulation_pb}) that $W^g_\x$ can be rewritten as the trace $W(\x)$ of a {\em simple} finite-dimensional functional space $W$. Meanwhile, we keep the specific notation.
For the moment, this set is explicitly characterized via the following proposition (which proof can be found in \Cref{appendix:prop2}).
\begin{proposition} \label{prop:2}
Assume that for each $y \in \mathcal{Y}$ the loss $\ell(z, y)$ is $\mathcal{C}^2$-differentiable with respect to $z \in \R^n$. 
For each $\x \in \Theta$ we have:
\vspace{-0.1cm}
    $$
    \vspace{-0.2cm}
    W_\x^g
    = \underset{(x, y) \in {\mathcal{X}}_\x \times \mathcal{Y}}{\linspan} \{ [\partial_\x g(\x,x)]^\top \nabla_z \ell (g(\x,x), y) \}
    $$
where $\mathcal{X}_\x$ is the set of data points $x$ such that $g(\cdot,x)$ is $\mathcal{C}^2$-differentiable in the neighborhood of $\x$. 
\end{proposition}

\begin{example}\label{example:firsttriviallaw}
As a first simple example, consider 
a two-layer {\em linear} neural network in dimension 1 (both for the input and output), with a single neuron. 
For such -- admittedly trivial -- architecture, the parameter is $\x = (u,v) \subseteq \R^2$ and the model writes $g(\x,x) = uvx$.
One can directly check that the function: $h(u, v) = u^2-v^2$ is {\em locally conserved on $\R^2$ for any data set}. Indeed
in that case $ \nabla h(u, v) = (2u, -2v)^\top \perp W_\x^g
    = \underset{(x, y) \in {\R \times \mathcal{Y}}}{\linspan} \{ (vx, ux)^\top \nabla_z \ell (g(\x , x), y) \} = \R \times (v, u)^\top$ given that the gradient $\nabla_z \ell (g(\x, x), y)$ is  an arbitrary scalar.
\end{example}

In this example we obtain a simple expression of $W_\x^g$, however in general cases it is not possible to obtain such a simple expression from \Cref{prop:2}. We will show that in some cases, it is possible to express $W_\x^g$ as the trace $W(\x)$ of a simple finite-dimensional space $W$ (cf. \Cref{theorem:reformulation_pb}).

\vspace{-0.1em}
\subsection{Reparametrization}
\vspace{-0.2em}
To make the mathematical analysis tractable and provide an algorithmic procedure to determine these functions, our fundamental hypothesis is that the \textbf{model} $g(\theta, x)$ 
can be (locally) factored via a \textbf{reparametrization} $\phi$ as $
f(\phi(\x),x)$. We require that the model $g(\theta, x)$ satisfies the following central assumption.

\begin{assumption}[Local reparameterization] \label{as:main_assumption}
There exists $d$ and $\phi \in \mathcal{C}^2(\Theta,\Rd)$ such that: for each parameter $\theta_0$ in the open set $\Theta \subseteq \RD$, for each $x \in \mathcal{X}$ such that $\x \mapsto 
    g(\x,x)$ is $\mathcal{C}^{2}$ in a neighborhood of $\x_0$\footnote{i.e., $x$ belongs to the set $\mathcal{X}_{\x_0}$, as defined in \Cref{prop:2}.}, there is a neighborhood $\ouv$ of $\theta_0$ and $f(\cdot, x) \in \mathcal{C}^2(\phi(\ouv), \R^n)$ such that
\begin{equation}
    \label{eq:elr-general}
    \vspace{-0.2cm}
      \forall \x \in \ouv, \quad  
     g(\theta,x) = 
     f(\phi(\theta), x).
\end{equation}
\end{assumption} 
Note that if the model $g(\cdot, x)$ is smooth on $\ouv$ then \eqref{eq:elr-general} is always satisfied with $\phi \coloneqq \texttt{id}$ and $f(\cdot, x) \coloneqq g(\cdot,x)$, yet this trivial factorization fails to capture the existence and number of conservation laws as studied in this paper. 
This suggests that, among all factorizations shaped as \eqref{eq:elr-general}, there may be a notion of an optimal one.
\begin{example}(Factorization for {\em linear} neural networks) \label{ex:param-linear}
In the two-layer case, with $r$ neurons, denoting $\x=(U,V) \in \RR^{n \times r} \times \RR^{m \times r}$ (so that $D=(n+m)r$), we can factorize  $g(\x,x) \coloneqq U V^\top x$ by the reparametrization 
$\phi(\x) \coloneqq U V^\top \in \R^{n \times m}$ using $f(
\phi,
x) = 
\phi \cdot x $. 
More generally for $q$ layers, with $\x = (U_1, \cdots, U_q)$, we can still factorize $g(\x,x) \coloneqq U_1 \cdots U_q x$ using $\phi(\x) \coloneqq U_1 \cdots U_q$ and the same $f$. 
This factorization is {\em globally} valid on $\Omega = \Theta = \RD$ since $f(\cdot, x)$ does not depend on $\x_0$. 
\end{example}
The notion of locality of the factorization~\eqref{eq:elr-general}
is illustrated by the next example.
%
%
\begin{example}[Factorization for two-layer ReLU networks] \label{ex:param-ReLU} 
Consider $g(\x,x) = \big( \sum_{j=1}^{r} u_{k, j} \sigma(\langle v_j, x \rangle + b_j)+ c_k\big)_{k = 1}^{n}
$, with $\sigma(t) \coloneqq \max(t,0)$ the ReLU activation function and $v_j \in \R^m$, $u_{k,j} \in \R$, $b_j, c_k \in \R$. 
Then, denoting $\x = (U,V, b, c)$ with $U = (u_{k,j})_{k, j} =: (u_1, \cdots, u_r) \in \R^{n \times r}$, $V = (v_1, \cdots, v_r) \in \R^{m \times r}$, $b = (b_1, \cdots, b_r)^\top \in \R^r$ and $c = (c_1, \cdots, c_n) \in \R^n$  (so that $D = (n+m+1)r + n$), we rewrite $g(\x,x) = \sum_{j=1}^r u_j \epsilon_{j, x} \left(v_j^\top x + b_j\right) + c$ 
where, given $x$, $\epsilon_{j, x} = \mathbb{1} ( v_j^\top x + b_j > 0 )$ is piecewise constant with respect to $\x$. 
Consider $\x^0 = (U^0, V^0, b^0, c^0) \in  \RD$ where $V^0 = (v_1^0, \cdots, v_r^0)$ and $b^0 = (b_1^0, \cdots, b_r^0)^\top$. Then the set $\mathcal{X}_{\x^0}$ introduced in \Cref{prop:2} is $\mathcal{X}_{\x^0} = \R^m - \cup_j \{ {v_j^0}^\top x  + b_j^0= 0 \}$. Let $x \in \mathcal{X}_{\x^0}$. Then on any domain $\ouv \subset \RD$ such that $\x^0 \in \ouv$ and $\epsilon_{j,x}(\x) \coloneqq \mathbb{1}(v_j^\top x + b_j> 0)$ is constant over $\x \in \Omega$, the model $g_{\x}(x)$ can be factorized by the reparametrization %
$\phi(\x) =  ((u_j v_j^\top, u_j b_j )_{j=1}^r, c)$. In particular, in the case without bias ($(b, c) = (0, 0)$), the reparametrization is defined by
$\phi(\x) = (\phi_j)_{j=1}^r$ where $\phi_j = \phi_j(\x) \coloneqq u_j v_j^\top \in \R^{n \times m}$ (here $d =rmn$) using $f(\phi,x) = \sum_{j} \epsilon_{j,x} \phi_{j}x$: the reparametrization $\phi(\theta)$ contains $r$ matrices of size $m \times n$ (each of rank at most one) associated to a ``local''  $f(\cdot, x)$ valid in a neighborhood of $\x$. 
A similar factorization is possible for deeper ReLU networks \cite{stock_embedding_2022}, as further discussed in the proof of \Cref{theorem:reformulation_pb} in \Cref{appendix_loss}.
\end{example}

Combining \Cref{prop:2} and using chain rules, we get a new characterization of $W^g_\x$:
\begin{proposition} \label{prop:3}
    Assume that the loss $\ell(z, y)$ is $\mathcal{C}^2$-differentiable with respect to $z$.  We recall (cf \eqref{eq:def_W_ell}) that $W^g_\x \coloneqq \cup_{\ouv \subseteq \Theta: \ouv\ \text{is open and}\ \ouv \ni \x} W_{\ouv}^g(\x).$ Under \Cref{as:main_assumption}, for all $\x \in \Theta$:
    \begin{equation} \label{eq:chain_rules}
    W^g_\x = \partial \phi(\x)^\top W^{f}_{\phi(\x)} 
    \end{equation}
    with
    $\partial \phi(\x) \in \mathbb{R}^{d \times D}$ the Jacobian of $\phi$ and $W^{f}_{\phi(\x)} \coloneqq \underset{(x, y) \in {\mathcal{X}}_\x \times \mathcal{Y}}{\linspan} \{ \partial f^x(\phi(\x))^\top \nabla_z \ell (g(\x , x), y) \}$, where
    $f^x(\cdot)\coloneqq f(\cdot, x)$.
\end{proposition}
We show in \Cref{sec:conservationlaw} that, under mild assumptions on the loss $\ell$, 
$W^{f}_{\phi(\x)} = \Rd$, so that \Cref{prop:3} yields
$W^g_\x = \range(\partial \phi (\x)^\top)$. Then by \Cref{coro:perp}, a function $h$ that is locally conserved on $\Theta$ for any data set is \emph{entirely characterized} via the kernel of $\partial \phi(\x)^\top$: $\partial \phi(\x)^\top \nabla h(\x) = 0$ for all $\x \in \Theta$.
The core of our analysis is then to analyze the 
(Lie algebraic) structure of 
$\range(\partial \phi(\cdot)^\top)$.

\vspace{-0.1em}
\subsection{From conserved functions to conservation laws}\label{sec:conservationlaw}
\vspace{-0.2em}

For linear and ReLU networks we show in \Cref{theorem:reformulation_pb} and \Cref{equiv} that under (mild) assumptions on the loss $\ell(\cdot,\cdot)$,
    being locally conserved on $\Theta$ for any data set (according to \Cref{def:locally_conserved_any}) is the same as being conserved (according to \Cref{def:conserved_through_flow}) on $\Theta$ through the {\em finite-dimensional} subspace 
    \vspace{-0.2cm}
\begin{equation} \label{eq:v-phi}
\vspace{-0.22cm}
    W_\phi \coloneqq \linspan \{\nabla \phi_1(\cdot), \cdots, \nabla \phi_{d}(\cdot) \}
    =
    \Big\{\x \mapsto \sum_i 
    a_i \nabla \phi_i(\x): (a_1,\ldots,a_{d}) \in \R^{d}\Big\}
\end{equation}
where we write $\partial \phi(\x)^\top = (\nabla \phi_1(\x), \cdots, \nabla \phi_{d}(\x)) \in \R^{D \times d}$, with $\nabla \phi_i \in \mathcal{C}^1(\Theta, \RD)$.

The following results (which proofs can be found in \Cref{appendix_loss}) establish that in some cases, the functions locally conserved for any data set are exactly the functions conserved through $W_\phi$.

\begin{lemma} \label{lemma:thm}
Assume that the loss $(z,y) \mapsto\ell(z,y)$ is $\mathcal{C}^2$-differentiable with respect to $z \in \R^n$ and satisfies the condition:
    \eq{ \label{eq:condition_loss}
    \underset{y \in \mathcal{Y}}{\linspan}\{\nabla_z \ell (z, y) \}= \R^n, \forall z \in \R^n.
    } 
    Then for linear neural networks (resp. for two-layer ReLU networks) and all 
    $\x \in \Theta$ we have  $W^{f}_{\phi(\x)} = \R^d$, with the reparametrization $\phi$ from \Cref{ex:param-linear} and $\Theta \coloneqq\RD$ (resp. with $\phi$ from \Cref{ex:param-ReLU} and $\Theta$ consisting of all parameter $\x$ of the network such that hidden neurons are associated to pairwise distinct  ``hyperplanes'', cf \Cref{appendix_loss} for details).
\end{lemma}
Condition~\eqref{eq:condition_loss} holds for classical losses $\ell$ (e.g. quadratic/logistic losses), as shown in  \Cref{lem:classicallosses} in \Cref{appendix_loss}. Note that the additional hypothesis of pairwise distinct hyperplanes for the two-layer ReLU case is a generic hypothesis and is usual (see e.g. the notion of twin neurons in \cite{stock_embedding_2022}). The tools from \Cref{appendix_loss} extend \Cref{theorem:reformulation_pb} beyond (deep) linear and shallow ReLU networks. An open problem is whether the conclusions of \Cref{lemma:thm} still hold for deep ReLU networks.

\begin{theorem} \label{theorem:reformulation_pb}
Under the same assumptions as in \Cref{lemma:thm}, we have that for linear neural networks, for all $\x \in \Theta \coloneqq \RD$:
    \begin{equation}
        W_\x^{g} = W_\phi(\x).
    \end{equation}
    The same result holds for two-layer ReLU networks with $\phi$ from \Cref{ex:param-ReLU} and $\Theta$ the (open) set of all parameters $\x$ such that hidden neurons are associated to pairwise distinct  ``hyperplanes''.
\end{theorem}
This means as claimed that for linear and two-layer ReLU networks, being locally conserved on $\Theta$ for any data set exactly means being conserved on $\Theta$ through the finite-dimensional functional space $W_\phi \subseteq 
\mathcal{C}^1(\Theta, \RD)$. 
This motivates the following definition
\begin{definition} \label{defconservationGF}
    A real-valued function $h$ is a \emph{conservation law of 
    $\phi$} if it is conserved through~$W_\phi$.
\end{definition}

\Cref{tracecharact} yields the following intermediate result.
\begin{proposition} \label{equiv}
    $h \in \mathcal{C}^1(\ouv, \mathbb{R})$ is a conservation law for $\phi$ 
    if and only if 
    $$\nabla h(\x) \perp \nabla \phi_j(\x),\ \forall\ \x \in \ouv,\ \forall j \in \{1,\ldots,d\}.$$
\end{proposition}
Thanks to \Cref{theorem:reformulation_pb}, the space $W_\phi$ defined in~\eqref{eq:v-phi} introduces a much simpler proxy to express $W_\x^g$ as a trace of a subset of $\mathcal{C}^1(\Theta, \RD)$. Moreover, when $\phi$ is $\mathcal{C}^\infty$, $W_\phi$ is a {\em finite-dimensional} space of \emph{infinitely smooth} functions on $\Theta$, and this will be crucial in \Cref{2layers} to provide a tractable scheme (i.e. operating in finite dimension) to compute the \emph{maximum number of independent} conservation laws, using the Lie algebra computations that will be described in \Cref{sec:LieAlgebra}. 

\begin{example}\label{example:firsttriviallawbis}
    Revisiting~\Cref{example:firsttriviallaw}, 
    the function to minimize is factorized by the reparametrization 
    $\phi: (u \in \R, v\in \R) \mapsto u  v \in \mathbb{R}$ with $\x \coloneqq (u, v)$. 
  We saw that $h((u, v)) \coloneqq u^2-v^2$ is conserved: and indeed $\langle \nabla h(u, v), \nabla \phi(u, v) \rangle = 2u v - 2v  u  = 0$, $\forall (u, v)$.
    \end{example}
    In this simple example, the characterization of 
    \Cref{equiv} gives a \textit{constructive} way to find such a conserved function: we only need to find a function $h$ such that $\langle \nabla h(u, v), \nabla \phi(u, v) \rangle = \langle \nabla h(u, v), (v,u)^\top \rangle = 0$. The situation becomes more complex in higher dimensions, since one needs to understand the interplay between the different vector fields in $W_\phi$. 
\vspace{-0.1cm}
\subsection{Constructibility of some conservation laws} \label{section:constructibility} 
\vspace{-0.2cm}
Observe that in \Cref{example:firsttriviallawbis} 
both the reparametrization $\phi$ and the conservation law $h$ are polynomials, a property that surprisingly systematically holds in all examples of interest in the paper, making it possible to {\em algorithmically} construct 
some conservation laws as detailed now. 

By~\Cref{equiv}, a function $h$ is a conservation law if it is in the kernel of the linear operator 
$h \in \mathcal{C}^{1}(\ouv, \R) \mapsto \left(\x \in \ouv \mapsto (\langle \nabla h (\x), \nabla \phi_i (\x) \rangle)_{i =1, \cdots, d}\right)$. Thus, one could look for conservation laws in a prescribed finite-dimensional space by projecting these equations in a basis (as in finite-element methods for PDEs).
Choosing the finite-dimensional subspace could be generally tricky, but for the linear and ReLU cases all known conservation laws are actually polynomial ``balancedness-type conditions'' \cite{Arora18a,Arora18b,Du}, see \Cref{sectionlinearNN}. 
In these cases, the vector fields in $W_\phi$ are also polynomials (because $\phi$ is polynomial, see \Cref{thm:LinearNetsAllF} and \Cref{thm:2layerReLUappendix} in \Cref{appendix_loss}), hence $\x \mapsto \langle \nabla h (\x), \nabla \phi_i (\x) \rangle$ is a polynomial too. This allows us to compute a basis of independent polynomial conservation laws of a given degree (to be freely chosen) for these cases, by simply focusing on the corresponding subspace of polynomials.
We coded the resulting equations in SageMath, and we found back on selected examples (see \Cref{appendix_numeric}) all existing known conservation laws both for ReLU and linear networks. Open-source code is available at \cite{marcotte:hal-04261339v1}.
\vspace{-0.1cm}
\subsection{Independent conserved functions}
\vspace{-0.2cm}
Having an algorithm to build conservation laws is nice, yet 
how can we know if we have built ``all'' laws? 
This requires first defining a notion of a ``maximal'' set of functions, which would in some sense be independent. 
This does not correspond to linear independence of the functions themselves (for instance, if $h$ is a conservation law, then so is $h^k$ 
for each $k \in \mathbb{N}$ but this does not add any other constraint), but rather to pointwise linear independence of their gradients. 
This notion of independence is closely related to the notion of ``functional independence'' studied in \cite{functionaldependence, functionaldependence-gradient}. For instance, it is shown in \cite{functionaldependence-gradient} that smooth functionally dependent functions are characterized by having dependent gradients everywhere.
This motivates the following definition. 

\begin{definition}
    A family of $N$ functions $(h_1, \cdots, h_N)$ conserved through $W \subset \mathcal{C}^1(\ouv, \RD)$ is said to be \textit{independent} if the vectors $(\nabla h_1(\x), \cdots, \nabla  h_N(\x))$ are linearly independent for all $\x \in \ouv$. 
\end{definition}

An immediate upper bound holds on the largest possible number $N$ of functionally independent functions $h_1,\ldots,h_N$ conserved through $W$: for $\x \in \Omega \subseteq \R^D$, the space $\linspan\{\nabla h_1(\x),\ldots, \nabla h_N(\x)\} \subseteq \RR^D$ is of dimension $N$ (by independence) and 
(by \Cref{tracecharact})
orthogonal to $W(\x)$. Thus, it is necessary to have 
$N \leq D-\dim W(\x)$.
As we will now see, 
this bound can be tight {\em under additional assumptions on $W$ related to Lie brackets} (corresponding to the so-called Frobenius theorem). This will in turn lead to a characterization of the maximum possible $N$.
%
\section{Conservation Laws using Lie Algebra} 
\label{sec:LieAlgebra}
The study of hyper-surfaces trapping the solution of ODEs is a recurring theme in control theory, since the existence of such surfaces is the basic obstruction of controllability of such systems~\cite{Bonnard}.
The basic result to study these surfaces is the so-called Frobenius theorem from differential calculus (See Section 1.4 of~\cite{isidori} for a good reference for this theorem).
It relates the existence of such surfaces, and their dimensions, to some differential condition involving so-called ``Lie brackets'' $[u,v]$ between pairs of vector fields (see \Cref{liealgebrabackground} below for a more detailed exposition of this operation). However, in most cases of practical interest (such as for instance matrix factorization), the Frobenius theorem is not suitable for a direct application to the space $W_\phi$ because its Lie bracket condition is not satisfied. To identify the number of independent conservation laws, one needs to consider the algebraic closure of $W_\phi$ under Lie brackets. The fundamental object of interest is thus the Lie algebra generated by the Jacobian vector fields, that we recall next. While this is only defined for vector fields with stronger smoothness assumption, the only consequence is that $\phi$ is required to be infinitely smooth, unlike the loss $\ell(\cdot,y)$ and the model $g(\cdot,x)$ that can be less smooth. All concretes examples of $\phi$ in this paper are polynomial hence indeed infinitely smooth.
\paragraph{Notations} Given a vector subspace of infinitely smooth vector fields $W \subseteq \mathcal{X}(\Theta) \coloneqq \mathcal{C}^{\infty}(\Theta, \RD)$, where $\Theta$ is an open subset of $\RD$, we recall (cf \Cref{tracecharact}) that its trace at some $\x$ is the subspace 
\eq{ \label{trace}
W(\x) \coloneqq \linspan \{ \chi(\x) : \chi \in W \} \subseteq \RD.
}
For each open subset $\ouv \subseteq \Theta$,  we introduce the subspace of $\mathcal{X}(\ouv)$:
$
W_{|\ouv} \coloneqq \{ \chi_{|\ouv}: \chi \in W \}.
$
\vspace{-0.1cm}
\subsection{Background on Lie algebra} 
\vspace{-0.2cm}
\label{liealgebrabackground}
A 
Lie algebra $A$ is a 
vector space endowed with
a bilinear map  $[\cdot, \cdot]$, called a Lie bracket, that verifies for all $X, Y, Z \in A$:
$
[X, X]= 0$ and the Jacobi identity: $
[X, [Y, Z]] + [Y, [Z, X]] + [Z, [X, Y]] = 0.
$

For the purpose of this article, the Lie algebra of interest is the set of infinitely smooth vector fields $\mathcal{X}(\Theta)$, endowed with the Lie bracket $[\cdot, \cdot]$ defined by
\begin{equation}\label{eq:def-lie-brac}
[\chi_1,\chi_2]:\quad 
\x \in \Theta \mapsto [\chi_1, \chi_2](\x)\coloneqq \partial \chi_1(\x) \chi_2(\x) - \partial \chi_2(\x) \chi_1(\x),
\end{equation}
with $\partial \chi (\theta) \in \R^{D \times D}$ the jacobian of $\chi$ at $\theta$.
The space $\R^{n \times n}$ of matrices is also a Lie algebra endowed with the Lie bracket
$[A, B] \coloneqq AB-BA.$   
This can be seen as a special case of~\eqref{eq:def-lie-brac} in the case of \emph{linear} vector fields, i.e. $\chi(\x)=A\x$. 
\paragraph{Generated Lie algebra} Let $A$ be a Lie algebra and let $W \subset A$ be a vector subspace of $A$. There exists a smallest Lie algebra that contains $W$. It is denoted $\lie(W)$ and called the generated Lie algebra of~$W$. The following proposition \cite[Definition 20]{Bonnard} constructively characterizes $\lie(W)$, where for vector subspaces $[W,W'] \coloneqq \{[\chi_1,\chi_2]: \chi_1 \in W, \chi_2 \in W'\}$, and $W+W' = \{\chi_1+\chi_2: \chi_1 \in W, \chi_2 \in W'\}$.

\begin{proposition} \label{buildingliealgebra}
    Given any vector subspace $W \subseteq A$ we have $\lie(W) = \bigcup_k W_k$ where:
\vspace{-0.5em}
     \begin{equation*} 
\left\{
    \begin{array}{ll}
        W_0 &\coloneqq W\\
        W_k &\coloneqq  W_{k-1} + [W_0, W_{k-1}]\ \text{ for }\ k \geq 1.
    \end{array}
\right.
\end{equation*}
\end{proposition}

We will see in \Cref{section_dim_lie_algebra} that the number of conservation laws is characterized by the dimension of the trace $\lie(W_\phi)(\x)$ defined in \eqref{trace}. The following lemma (proved in \Cref{dimstagnates})
 gives a stopping criterion to algorithmically determine this dimension (see \Cref{subsection-method} for the algorithm). 
\begin{lemma} \label{lemma-stagnates}
Given $\x \in \Theta$, if for a given $i$, $\vdim W_{i+1}(\x') = \vdim W_i(\x)$ for every $\theta'$ in 
a neighborhood of $\x$, then there exists a neighborhood $\Omega$ of $\theta$ such that 
$W_k(\x')= W_i (\x')$ for all $\x' \in \Omega$ and $k\geq i$,
where the $V_i$ are defined by \Cref{buildingliealgebra}. Thus $\lie(W)(\x')=  W_i (\x')$ for all $\x' \in \Omega$. In particular, the dimension of the trace of $\lie(W)$ is locally constant and equal to the dimension of $W_i(\x)$.
\end{lemma}
\vspace{-0.1cm}
\subsection{Number of conservation laws} 
\vspace{-0.2cm}
\label{section_dim_lie_algebra}
The following theorem uses the Lie algebra generated by $W_\phi$ to characterize the number of conservation laws. The proof of this result 
is based on two successive uses of the Frobenius theorem and can be found in \Cref{appendix_main_theorem} (where we also recall Frobenius theorem for the sake of completeness).

\begin{theorem} \label{mainthm}
If $\vdim (\lie(W_\phi)(\x))$ is locally constant then each $\x \in \ouv \subseteq \RD$ admits a neighborhood $\Omega'$ such that there are $D-
\vdim (\lie(W_\phi)(\x))$ (and no more) independent conserved functions through ${W_\phi}_{\mid \Omega'}$, i.e., there are $D-\vdim (\lie(W_\phi)(\x))$ independent conservation laws of $\phi$ on $\ouv'$. 
\end{theorem}
\begin{remark}
The proof of the Frobenius theorem (and therefore of our generalization \Cref{mainthm}) is actually constructive. From a given $\phi$, conservation laws are obtained in the proof by integrating in time (\textit{i.e.} solving an advection equation) the vector fields belonging to $W_\phi$. 
Unfortunately, this cannot be achieved in \textit{closed form} in general, but in small dimensions, this could be carried out numerically (to compute approximate discretized laws on a grid or approximate them using parametric functions such as Fourier expansions or neural networks). 
\end{remark}

\vspace{-0.2cm}
A fundamental aspect of \Cref{mainthm} 
is to rely only on the \emph{dimension of the trace} of the Lie algebra associated with the finite-dimensional vector space $W_\phi$. 
Yet, even if $W_\phi$ is finite-dimensional, it might be the case that $\lie(W_\phi)$ itself remains infinite-dimensional. 
Nevertheless, what matters is not the dimension of $\lie(W_\phi)$, but that of {\em its trace} $\lie(W_\phi)(\x)$, which is \emph{always} finite (and potentially much smaller that $\dim \lie(W_\phi)$ even when the latter is finite) and computationally tractable thanks to \Cref{lemma-stagnates} as detailed in \Cref{subsection-method}.
In section~\ref{2layers} we work out the example of matrix factorization, a non-trivial case where the full Lie algebra $\lie(W_\phi)$ itself remains finite-dimensional. 

\Cref{mainthm} 
requires that the dimension 
of the trace at $\x$ of the Lie algebra is locally constant. 
This is a technical assumption, which typically holds outside a set of pathological points. A good example is once again matrix factorization, where we show in Section~\ref{2layers} that this condition holds generically.  
\vspace{-0.2cm}
\subsection{Method and algorithm, with examples} 
\vspace{-0.2cm}
\label{subsection-method}
Given a reparametrization $\phi$ for the architectures to train, to determine the number of independent conservation laws of $\phi$, we leverage the characterization~\ref{buildingliealgebra} to algorithmically compute $\vdim (\lie(W_\phi)(\x))$ using an iterative construction of bases for the subspaces $W_k$ starting from $W_0 \coloneqq W_\phi$, and stopping as soon as the dimension stagnates thanks to Lemma \ref{lemma-stagnates}.
Our open-sourced code is available at \cite{marcotte:hal-04261339v1} and uses SageMath. As we now show, this algorithmic principle allows to fully work out certain settings where the stopping criterion of \Cref{lemma-stagnates} is reached at the first step ($i=0$) or the second one ($i=1$).
\Cref{section:numerics} also discusses its numerical use for an empirical investigation of broader settings.

\vspace{-0.2em}
\paragraph{Example where the iterations of \Cref{lemma-stagnates} stop at the first step.} This corresponds to the case where $\lie W_\phi (\x) = W_1(\x) = W_0(\x) \coloneqq W_\phi(\x)$ on $\ouv$. This is the case if and only if $W_\phi$ satisfies that
\begin{equation} \label{eq:frob-crochets}
   [\chi_1, \chi_2](\x) \coloneqq \partial \chi_1 (\x) \chi_2 (\x) - \partial \chi_2(\x) \chi_1(\x)  \in W_\phi(\x),
   \quad \text{for all}\ \chi_1, \chi_2 \in W_\phi\ \text{and all}\ \x \in \ouv.
\end{equation}
i.e., when Frobenius Theorem (see \Cref{frobenius} in \Cref{appendix_main_theorem}) applies directly.
The first example is a follow-up to \Cref{ex:param-ReLU}.
\begin{example}[two-layer ReLU networks without bias] 
\label{shallowNN}
Consider $\x = (U,V)$ with $U \in \R^{n\times r}, V \in \R^{m \times r}$, $n, m,r \geq 1$ (so that $D=(n+m)r$), and the reparametrization
$\phi(\x) \coloneqq (u_i v_i^\top)_{i=1, \cdots, r} \in \R^{n \times m \times r},
$ where $U=(u_1; \cdots; u_r)$ and $ V=(v_1; \cdots; v_r)$.
As detailed in \Cref{appA1}, 
since $\phi(\x)$ is a collection of $r$ rank-one $n\times m$ matrices, $\vdim (W_\phi(\x)) = \mathtt{rank} \partial \phi(\x)=(n+m-1)r$
is constant 
on the domain $\Omega$ such that $u_i, v_j \neq 0$, and  $W_\phi$ satisfies~\eqref{eq:frob-crochets}, hence by~\Cref{mainthm} each $\x$ has a neighborhood $\ouv'$ such that there exists $r$ (and no more) independent conserved function through ${W_\phi}_{|\ouv'}$. The $r$ known conserved functions  \cite{Du} given by $h_i: (U, V) \mapsto \|u_i\|^2 - \|v_i\|^2$, $i=1, \cdots, r$, are independent, hence they are complete.
\end{example}
\vspace{-0.2em}
\paragraph{Example where the iterations of \Cref{lemma-stagnates} stop at the second step (but not the first one).} Our primary example is matrix factorization, as a follow-up to \Cref{ex:param-linear}.
\begin{example}[two-layer {\em linear} neural networks] \label{ex-contre-ex-frobenius}
With $\x = (U,V)$, where $(U \in \R^{n\times r}, V \in \R^{m \times r})$ the reparameterization 
$
\phi(\x) \coloneqq U V^\top \in \R^{n \times m}
$ (here $d=nm$)
factorizes the functions minimized during the training of linear two-layer neural networks (see \Cref{ex:param-linear}). 
As shown in \Cref{contre-ex-frobenius}, condition \eqref{eq:frob-crochets} is not satisfied
when $r > 1$ and $\text{max}(n, m) > 1$.  Thus, the stopping criterion of \Cref{lemma-stagnates} is not satisfied at the first step. However, as detailed in \Cref{lieW} in \Cref{appendix:2layerLNN}, $(W_\phi)_1 = (W_\phi)_2 = \lie (W_\phi)$, hence the iterations of \Cref{lemma-stagnates} stop at the second step.
\end{example}
We complete this example in the next section by showing (\Cref{nbatteint}) that known conservation laws are indeed complete. Whether known conservation laws remain valid and/or \emph{complete} in this settings and extended ones is further studied in \Cref{sectionlinearNN,annex:ExamplesSection3}. 
\vspace{-0.1em}
\subsection{Application: recasting over-parameterized flows as low-dimensional Riemannian flows}
\label{sec:riemanian-flow}
\vspace{-0.2em}
As we now show,
one striking application of \Cref{mainthm} 
(in simple cases where $\vdim (W_\phi( \x)) = \vdim (\lie W_\phi( \x))$ is constant on $\Omega$, i.e., $\mathtt{rank}(\partial \phi(\x))$ is constant on $\Omega$ and $W_\phi$ satisfies~\eqref{eq:frob-crochets})
is to fully rewrite the high-dimensional flow $\x(t) \in \RD$ as a low-dimensional flow on $z(t) \coloneqq \phi(\x(t)) \in \Rd$, where this flow is associated with a Riemannian metric tensor $M$ that is induced by $\phi$ and depends on the initialization $\x_{\text{init}}$.
We insist on the fact that this is only possible in very specific cases, but this phenomenon is underlying many existing works that aim at writing in closed form the implicit bias associated with some training dynamics (see \Cref{Intro} for some relevant literature. Our analysis sheds some light on cases where this is possible, as shown in the next proposition.

\begin{proposition} \label{lowerdimflow}
Assume that 
$\mathtt{rank}(\partial\phi(\x))$ 
is constant on $\Omega$ 
 and that $W_\phi$ satisfies \eqref{eq:frob-crochets}.
If $\x(t) \in \RD$ satisfies the ODE~\eqref{gradientflow} where $\x_{\text{init}} \in \Omega$, then there is $0<T^{\star}_{\theta_\mathtt{init}}\leq T_{\theta_\mathtt{init}}$ such that $z(t) \coloneqq \phi(\x(t)) \in 
\Rd$ satisfies the ODE 
 \begin{equation} \label{transfertGF}
        \overset{.}{z}(t) =  -M(z(t), \x_{\text{init}}) \nabla f(z(t)) \quad \mbox{for all } t \in [0,
        T^{\star}_{\x_{\mathtt{init}}}
        ), 
        \text{ with } z(0) = \phi(\x_{\text{init}}),
\end{equation}
where $M(z(t),\x_{\text{init}}) \in \R^{d \times d} $ is a symmetric positive semi-definite matrix. 
\end{proposition}

See \Cref{appendix-lowerdim} for a proof.
Revisiting \Cref{shallowNN} leads to the following analytic example.
\begin{example}\label{ex:riemannian1} Given the reparametrization 
$\phi: (u \in \R^*, v \in \Rd) \mapsto u  v \in \Rd$, the variable $z \coloneqq uv$ satisfies  \eqref{transfertGF} with 
$
    M(z, \x_{\text{init}}) =  \|z\|_\delta  \mathrm{I}_{d} 
    +  
     \|z\|_\delta ^{-1} z z^\top,
$
with $\|z\|_\delta \coloneqq \delta + \sqrt{\delta ^{2} +  \| z \| ^2}$, $\delta \coloneqq 1/2(u_{\text{init}}^2-\| v _{\text{init}} \|^2)$.
\end{example}
Another analytic example is discussed in \Cref{appendix-lowerdim}.
In light of these results, an interesting perspective is to better understand the dependance of the Riemannian metric with respect to initialization, to possibly guide the choice of initialization for better convergence dynamics. 

Note that the metric $M(z,\x_{\text{init}})$ can have a kernel. 
Indeed, in practice, while $\phi$ is a function from $\RD$ to $\Rd$, the dimensions often satisfy 
$\mathtt{rank} \partial \phi(\x) < \min(d,D)$, i.e., $\phi(\x)$ lives in a manifold of lower dimension. 
The evolution~\eqref{transfertGF} should then be understood as a flow on this manifold. The kernel of $M(z,\x_{\text{init}})$ is orthogonal to the tangent space at $z$ of this manifold. 
\section{Conservation Laws for Linear and ReLU Neural Networks} \label{sectionlinearNN}
\vspace{-0.1cm}
To showcase the impact of our results, we show how they can be used to determine whether known conservation laws for linear (resp. ReLU) neural networks are complete, and to recover these laws \emph{algorithmically} using reparametrizations $\phi$ adapted to these two settings. Concretely, we study the conservation laws for neural networks with $q$ layers, and either a linear or ReLU activation, with an emphasis
on $q=2$. We write $\x = (U_1, \cdots, U_q)$ with $U_i \in \R^{n_{i-1} \times n_{i}}$ the weight matrices 
and we assume that $\x $ satisfies the gradient flow \eqref{gradientflow}.
In the linear case the reparametrization is $\phi_{\mathtt{Lin}}(\x) 
\coloneqq  U_1  \cdots  U_q$. For ReLU networks, we use the (polynomial) reparametrization $\phi_{\mathtt{ReLu}}$ of \cite[Definition 6]{stock_embedding_2022}, which is defined for any (deep) feedforward ReLU network, with or without bias. In the simplified setting of networks without biases it reads explicitly as:
\begin{equation}
\phi_{\mathtt{ReLu}}(U_1, \cdots, U_q) \coloneqq
\Big(U_1[:, j_1]  U_2[j_1, j_2] \cdots U_{q-1}[j_{q-2}, j_{q-1}] U_{q}[j_{q-1},:]\Big)_{j_1, \cdots, j_{q-1}}
\end{equation}
with $U[i, j]$ the $(i, j)$-th entry of $U$. This covers $\phi(\x) \coloneqq (u_j v_j^\top)_{j=1}^r \in \R^{n \times m \times r}$ from \Cref{ex:param-ReLU}.

Some conservation laws are known for the linear case $\phi_{\mathtt{Lin}}$ 
 \cite{Arora18a, Arora18b}
and for the ReLu case $\phi_{\mathtt{ReLu}}$ \cite{Du}.

\begin{proposition}[{ \cite{Arora18a, Arora18b,Du}
}] \label{conservation}
If $\x \coloneqq(U_1, \cdots, U_q)$ satisfies the gradient flow \eqref{gradientflow}, then for each $i = 1, \cdots, q-1$ the function  
$\x \mapsto  U_{i}^\top U_{i} - U_{i+1} U_{i+1}^\top$ (resp. the function $\x \mapsto  \text{diag} \left(U_{i}^\top U_{i} - U_{i+1} U_{i+1}^\top  \right)$)
defines $n_i \times (n_i +1)/2$ conservation laws for $\phi_{\mathtt{Lin}}$ (resp. $n_i$ conservation laws for $\phi_{\mathtt{ReLu}}$).
\end{proposition}

Proposition \ref{conservation} defines $\sum_{i=1}^{q-1} n_i \times (n_i +1)/2$ conserved functions for the linear case. 
In general they are \emph{not} independent, and we give below in Proposition~\ref{classicinv}, for the case of $q=2$, the \emph{exact} number of independent conservation laws among these particular laws.
Establishing whether there are other (previously unknown) conservation laws is an open problem for $q>2$. We already answered negatively to this question in the two-layer ReLu case without bias (See \Cref{shallowNN}). In the following Section (\Cref{nbatteint}), we show the same result in the linear case $q=2$. Numerical computations suggest this is still the case for deeper linear and ReLU networks as detailed in \Cref{section:numerics}.
\vspace{-0.1em}
\subsection{The matrix factorization case ($q=2$)} \label{2layers}
\vspace{-0.2em}
To simplify the analysis when $q=2$, we rewrite $\x = (U,V)$ as a vertical matrix concatenation denoted $(U; V) \in \mathbb{R}^{(n+m) \times r}$, and $\phi(\x) = \phi_{\mathtt{Lin}}(\x) = 
U V^\top \in \R^{n \times m}$.

\paragraph{How many independent conserved functions are already known?}  
The following proposition refines Proposition~\ref{conservation} for $q=2$ by detailing how many \emph{independent} conservation laws are already known. See Appendix \ref{appC} for a proof.

\begin{proposition} \label{classicinv}
     Consider 
     $\Psi: \x = (U; V) \mapsto U^\top U - V^\top V \in \R^{r \times r}$ and assume that $(U; V)$ has full rank noted $\mathtt{rk}$. Then the function $\Psi$ gives $\mathtt{rk} \cdot (2r + 1-\mathtt{rk})/2$ 
     independent conservation laws.
\end{proposition}

\paragraph{There exist no more independent conservation laws.}
We now come to the core of the analysis, which consists in actually computing $\lie(W_\phi)$ as well as its traces $\lie(W_\phi)(\x)$ in the matrix factorization case.
The crux of the analysis, which enables us to fully work out theoretically the case $q=2$, is that $W_\phi$ is composed of \emph{linear} vector fields (that are explicitly characterized in \Cref{linear_vf} in \Cref{appendix:2layerLNN}),  the Lie bracket between two linear fields 
being itself linear and explicitly characterized with skew matrices, see \Cref{lieW} in \Cref{appendix:2layerLNN}. Eventually, what we need to compute is the dimension of the trace $\lie(W_\phi) (U, V)$ for any $(U, V)$. We prove the following in \Cref{appendix:2layerLNN}. 

\begin{proposition} \label{dimlie}
If $(U; V) \in \R^{(n+m) \times r}$ has full rank noted $\mathtt{rk}$, then: $\vdim (\lie(W_\phi) (U; V)) \!=\! (n+m)r -  \mathtt{rk}(2r + 1-\mathtt{rk})/2$.
\end{proposition}

With this explicit characterization of the trace of the generated Lie algebra and \Cref{classicinv}, we conclude that \Cref{conservation}  has indeed exhausted the list of independent conservation laws. 

\begin{corollary} \label{nbatteint}
    If $(U; V)$ has full rank, then all conservation laws are given by $\Psi: (U, V) \mapsto U^\top U - V^\top V$. In particular, there exist no more \emph{independent} conservation laws.
\end{corollary}

\vspace{-0.1em}
\subsection{Numerical guarantees in the general case} \label{section:numerics}
\vspace{-0.2em}
The expressions derived in the previous section are specific to the linear case $q=2$. For deeper linear networks and for ReLU networks, the vector fields in $W_\phi$ are non-linear polynomials, and computing Lie brackets of such fields can increase the degree, which could potentially make the generated Lie algebra infinite-dimensional. One can however use \Cref{lemma-stagnates} and stop as soon as $\vdim \left(({W_\phi})_k(\x)\right)$ stagnates. Numerically comparing this dimension with the number $N$ of independent conserved functions known in the literature (predicted by \Cref{conservation}) on a sample of depths/widths of small size, we empirically confirmed that there are no more conservation laws than the ones already known for deeper linear networks and for ReLU networks too (see \Cref{appendix_numeric} for details).  Our code is open-sourced and is available at \cite{marcotte:hal-04261339v1}. It is worth mentioning again that in all tested cases $\phi$ is polynomial, and there is a maximum set of conservation laws that are also polynomial, which are found algorithmically (as detailed in \Cref{section:constructibility}). 

\section*{Conclusion}
In this article, we proposed a constructive program for determining the number of conservation laws. 
An important avenue for future work is the consideration of more general classes of architectures, such as deep convolutional networks, normalization, and attention layers.
Note that while we focus in this article on gradient flows, our theory can be applied to any space of displacements in place of $W_\phi$. This could be used to study conservation laws for flows with higher order time derivatives, for instance gradient descent with momentum, by lifting the flow to a higher dimensional phase space. 
A limitation that warrants further study is that our theory is restricted to continuous time gradient flow. Gradient descent with finite step size, as opposed to continuous flows, disrupts exact conservation. The study of approximate conservation presents an interesting avenue for future work.

\section*{Acknowledgement} 
The work of G. Peyré was supported by the European Research Council (ERC project NORIA) and the French government under management of Agence Nationale de la Recherche as part of the ``Investissements d’avenir'' program, reference ANR-19-P3IA-0001 (PRAIRIE 3IA Institute). The work of R. Gribonval was partially supported by the AllegroAssai ANR project ANR-19-CHIA-0009.
We thank Thomas Bouchet for introducing us to SageMath, as well as Léo Grinsztajn for helpful feedbacks regarding the numerics. We thank Pierre Ablin and Raphaël Barboni for comments on a draft of this paper. We also thank the anonymous reviewers for their fruitful feedback.

\bibliography{biblio}

\begin{thebibliography}{10}

\bibitem{Arora18a}
{\sc S.~Arora, N.~Cohen, N.~Golowich, and W.~Hu}, {\em A convergence analysis of gradient descent for deep linear neural networks}, arXiv preprint arXiv:1810.02281,  (2018).

\bibitem{Arora18b}
{\sc S.~Arora, N.~Cohen, and E.~Hazan}, {\em On the optimization of deep networks: Implicit acceleration by overparameterization}, in Int. Conf. on Machine Learning, PMLR, 2018, pp.~244--253.

\bibitem{azulay21}
{\sc S.~Azulay, E.~Moroshko, M.~S. Nacson, B.~E. Woodworth, N.~Srebro, A.~Globerson, and D.~Soudry}, {\em On the implicit bias of initialization shape: Beyond infinitesimal mirror descent}, in Proceedings of the 38th International Conference on Machine Learning, M.~Meila and T.~Zhang, eds., vol.~139 of Proceedings of Machine Learning Research, PMLR, 18--24 Jul 2021, pp.~468--477.

\bibitem{Bah}
{\sc B.~Bah, H.~Rauhut, U.~Terstiege, and M.~Westdickenberg}, {\em Learning deep linear neural networks: Riemannian gradient flows and convergence to global minimizers}, Information and Inference: A Journal of the IMA, 11 (2022), pp.~307--353.

\bibitem{belkin}
{\sc M.~Belkin, D.~Hsu, S.~Ma, and S.~Mandal}, {\em Reconciling modern machine-learning practice and the classical bias--variance trade-off}, Proc. of the National Academy of Sciences, 116 (2019), pp.~15849--15854.

\bibitem{Bonnard}
{\sc B.~Bonnard, M.~Chyba, and J.~Rouot}, {\em {Geometric and Numerical Optimal Control - Application to Swimming at Low Reynolds Number and Magnetic Resonance Imaging}}, SpringerBriefs in Mathematics, {Springer Int. Publishing}, 2018.

\bibitem{functionaldependence}
{\sc A.~B. Brown}, {\em Functional dependence}, Transactions of the American Mathematical Society, 38 (1935), pp.~379--394.

\bibitem{chizat}
{\sc L.~Chizat and F.~Bach}, {\em Implicit bias of gradient descent for wide two-layer neural networks trained with the logistic loss}, in Conf. on Learning Theory, PMLR, 2020, pp.~1305--1338.

\bibitem{Du}
{\sc S.~S. Du, W.~Hu, and J.~D. Lee}, {\em Algorithmic regularization in learning deep homogeneous models: Layers are automatically balanced}, Adv. in Neural Inf. Proc. Systems, 31 (2018).

\bibitem{gunasekar}
{\sc S.~Gunasekar, J.~Lee, D.~Soudry, and N.~Srebro}, {\em Characterizing implicit bias in terms of optimization geometry}, in Int. Conf. on Machine Learning, PMLR, 2018, pp.~1832--1841.

\bibitem{gunasekar2017implicit}
{\sc S.~Gunasekar, B.~E. Woodworth, S.~Bhojanapalli, B.~Neyshabur, and N.~Srebro}, {\em Implicit regularization in matrix factorization}, Adv. in Neural Inf. Proc. Systems, 30 (2017).

\bibitem{głuch2021noether}
{\sc G.~Głuch and R.~Urbanke}, {\em Noether: The more things change, the more stay the same}, 2021.

\bibitem{isidori}
{\sc A.~Isidori}, {\em Nonlinear system control}, New York: Springer Verlag, 61 (1995), pp.~225--236.

\bibitem{jigradient}
{\sc Z.~Ji, M.~Dud{\'\i}k, R.~E. Schapire, and M.~Telgarsky}, {\em Gradient descent follows the regularization path for general losses}, in Conf. on Learning Theory, PMLR, 2020, pp.~2109--2136.

\bibitem{Ji}
{\sc Z.~Ji and M.~Telgarsky}, {\em Gradient descent aligns the layers of deep linear networks}, arXiv preprint arXiv:1810.02032,  (2018).

\bibitem{Kunin}
{\sc D.~Kunin, J.~Sagastuy-Brena, S.~Ganguli, D.~L. Yamins, and H.~Tanaka}, {\em Neural mechanics: Symmetry and broken conservation laws in deep learning dynamics}, arXiv preprint arXiv:2012.04728,  (2020).

\bibitem{li22}
{\sc Z.~Li, T.~Wang, J.~D. Lee, and S.~Arora}, {\em Implicit bias of gradient descent on reparametrized models: On equivalence to mirror descent}, in Advances in Neural Information Processing Systems, S.~Koyejo, S.~Mohamed, A.~Agarwal, D.~Belgrave, K.~Cho, and A.~Oh, eds., vol.~35, Curran Associates, Inc., 2022, pp.~34626--34640.

\bibitem{marcotte:hal-04261339v1}
{\sc S.~Marcotte, R.~Gribonval, and G.~Peyr{\'e}}, {\em {Code for reproducible research. Abide by the Law and Follow the Flow: Conservation Laws for Gradient Flows}}, Oct. 2023.

\bibitem{Min}
{\sc H.~Min, S.~Tarmoun, R.~Vidal, and E.~Mallada}, {\em On the explicit role of initialization on the convergence and implicit bias of overparametrized linear networks}, in Int. Conf. on Machine Learning, PMLR, 2021, pp.~7760--7768.

\bibitem{functionaldependence-gradient}
{\sc W.~F. Newns}, {\em Functional dependence}, The American Mathematical Monthly, 74 (1967), pp.~911--920.

\bibitem{neyshaburthesis}
{\sc B.~Neyshabur}, {\em Implicit regularization in deep learning}, arXiv preprint arXiv:1709.01953,  (2017).

\bibitem{neyshabur}
{\sc B.~Neyshabur, R.~Tomioka, and N.~Srebro}, {\em In search of the real inductive bias: On the role of implicit regularization in deep learning}, arXiv preprint arXiv:1412.6614,  (2014).

\bibitem{Noether1918}
{\sc E.~Noether}, {\em Invariante variationsprobleme}, Nachrichten von der Gesellschaft der Wissenschaften zu Göttingen, Mathematisch-Physikalische Klasse, 1918 (1918), pp.~235--257.

\bibitem{Saxe}
{\sc A.~M. Saxe, J.~L. McClelland, and S.~Ganguli}, {\em Exact solutions to the nonlinear dynamics of learning in deep linear neural networks}, arXiv preprint arXiv:1312.6120,  (2013).

\bibitem{MLtheory}
{\sc S.~Shalev-Shwartz and S.~Ben-David}, {\em Understanding machine learning: From theory to algorithms}, Cambridge university press, 2014.

\bibitem{soudry}
{\sc D.~Soudry, E.~Hoffer, M.~S. Nacson, S.~Gunasekar, and N.~Srebro}, {\em The implicit bias of gradient descent on separable data}, The Journal of Machine Learning Research, 19 (2018), pp.~2822--2878.

\bibitem{stock_embedding_2022}
{\sc P.~Stock and R.~Gribonval}, {\em An {Embedding} of {ReLU} {Networks} and an {Analysis} of their {Identifiability}}, Constructive Approximation,  (2022).
\newblock Publisher: Springer Verlag.

\bibitem{Tanaka}
{\sc H.~Tanaka and D.~Kunin}, {\em Noether’s learning dynamics: Role of symmetry breaking in neural networks}, Adv. in Neural Inf. Proc. Systems, 34 (2021), pp.~25646--25660.

\bibitem{Tarmoun}
{\sc S.~Tarmoun, G.~Franca, B.~D. Haeffele, and R.~Vidal}, {\em Understanding the dynamics of gradient flow in overparameterized linear models}, in Int. Conf. on Machine Learning, PMLR, 2021, pp.~10153--10161.

\bibitem{sagemath}
{\sc {The Sage Developers}}, {\em {S}ageMath, the {S}age {M}athematics {S}oftware {S}ystem ({V}ersion 9.7)}, 2022.
\newblock {\tt https://www.sagemath.org}.

\bibitem{zhang}
{\sc C.~Zhang, S.~Bengio, M.~Hardt, B.~Recht, and O.~Vinyals}, {\em Understanding deep learning requires rethinking generalization}, in Int. Conf. on Learning Representations, 2017.

\bibitem{zhao}
{\sc B.~Zhao, I.~Ganev, R.~Walters, R.~Yu, and N.~Dehmamy}, {\em Symmetries, flat minima, and the conserved quantities of gradient flow}, arXiv preprint arXiv:2210.17216,  (2022).

\end{thebibliography}
\bibliographystyle{siam}
\newpage
\appendix

\section{Proof of \Cref{tracecharact}} \label{appendix-orth}
\Cref{tracecharact} is a direct consequence of the following lemma (remember that $\nabla h(\theta) = [\partial h(\theta)]^\top$).
\begin{lemma}[
Smooth functions conserved through a given flow.] \label{prop:conservedonefield}
Given $\chi \in \mathcal{C}^1(\ouv, \RD)$, a function $h \in \mathcal{C}^1(\ouv, \mathbb{R})$ is conserved through the flow induced by $\chi$ 
     if and only if $\partial h (\x) \chi(\x) = 0$ for all $\x \in \ouv$.
\end{lemma}
\begin{proof}
    Assume that  $\partial h (\x) \chi(\x) = 0$ for all $\x \in \ouv$. Then for all $\x_{\text{init}} \in \ouv$ and for all $t \in (0, T_{\x_{\text{init}}}):$
    \begin{align*}
\frac{\mathrm{d}}{\mathrm{d}t}  h(\x(t, \x_{\text{init}})) &= \partial h(\x(t,  \x_{\text{init}}))  \overset{.}{\x}(t, \x_{\text{init}}) 
= \partial h(\x(t, \x_{\text{init}})) \chi(\x(t, \x_{\text{init}}))=  0.
\end{align*}
Thus: $h(\x(t, \x_{\text{init}}))= h(\x_{\text{init}})$, {\em i.e.}, $h$ is conserved through~$\chi$.
Conversely, assume that there exists $\x_0 \in \ouv$ such that $\partial h(\x_0) \chi(\x_0) \neq 0$. Then by continuity of $\x \in \ouv \mapsto \partial h(\x)\chi(\x)$, there exists $r > 0$ such that $\partial h(\x) \chi(\x) \neq 0$ on $B(\x_0, r)$. With $\x_{\text{init}} = \x_0$ 
by continuity of $t \mapsto \x(t, \x_{\text{init}})$, there exists $\epsilon > 0$, such that for all $t< \epsilon$, $\x(t, \x_{\text{init}}) \in B(\x_0, r)$. Then for all $t \in (0, \epsilon)$:$
\frac{\mathrm{d}}{\mathrm{d}t}  h(\x(t,\x_{\text{init}})) = \partial h(\x(t, \x_{\text{init}})) \chi(\x(t, \x_{\text{init}})) \neq  0,$
hence $h$ is not conserved through the flow induced by $\chi$.
\end{proof}

\section{Proof of \Cref{prop:2}} \label{appendix:prop2}
\begin{proposition} 
Assume that for each $y \in \mathcal{Y}$ the loss $\ell(z, y)$ is $\mathcal{C}^2$-differentiable with respect to $z \in \R^n$. 
For each $\x \in \Theta$ we have:
    $$
    W_\x^g
    = \underset{(x, y) \in {\mathcal{X}}_\x \times \mathcal{Y}}{\linspan} \{ [\partial_\x g(\x,x)]^\top \nabla_z \ell (g(\x,x), y) \}
    $$
where $\mathcal{X}_\x$ is the set of data points $x$ such that $g(\cdot,x)$ is $\mathcal{C}^2$-differentiable in the neighborhood of $\x$. 
\end{proposition}

\begin{proof}
  Let us first show the direct inclusion. Let $\ouv \subseteq \Theta$ be a neighborhood of $\x$ and let $\chi \in W_{\ouv}^g$. Let us show that $\chi(\x) \in  \underset{(x, y) \in {\mathcal{X}}_\x \times \mathcal{Y}}{\linspan} \{ [\partial_\x g(\x,x)]^\top \nabla_z \ell (g(\x,x), y) \}$.
  As $\chi \in W_{\ouv}^g$, there exist $X = (x_i)_i, Y = (y_i)_i$ such that  $\forall i\ g(\cdot, x_i) \in \mathcal{C}^{2}(\Omega,\R)$ (and thus $x_i \in \mathcal{X}_\x$) and $\chi(\cdot) = \nabla \mathcal{E}_{X, Y}(\cdot) \in \mathcal{C}^1(\ouv, \RD)$ (cf \eqref{eq:W_ell}).
  Moreover, for each $\x' \in \ouv$, by chain rules and~\eqref{eq:erm}, we have:
  $$
  \nabla \mathcal{E}_{X, Y}(\x') = \sum_i [\partial_\x g(\x', x_i)]^\top \nabla_z \ell (g(\x', x_i), y_i),
  $$
  where $x_i \in \mathcal{X}_\x$. 
  Thus $\chi(\x) \in  \underset{(x, y) \in {\mathcal{X}}_\x \times \mathcal{Y}}{\linspan} \{ [\partial_\x g(\x,x)]^\top \nabla_z \ell (g(\x,x), y) \}$. This leads to the direct inclusion.

  Now let us show the converse inclusion. Let $(x, y) \in \mathcal{X}_\x \times \mathcal{Y}$. 
  Let us show that $[\partial_\x g(\x, x)]^\top \nabla_z \ell (g(\x, x), y) \in 
  W_\x^g$.
  By definition of $\mathcal{X}_\x$, there exists a neighborhood $\ouv$ of $\x$ such that 
  $g(\cdot, x)  \in \mathcal{C}^2(\ouv, \RD)$. By taking $ X = x$ and $Y = y$ (\textit{i.e.} a data set of one feature and one target), 
  one has still by chain rules $\nabla \mathcal{E}_{X, Y} (\cdot) = [\partial_\x g(\cdot, x)]^\top \nabla_z \ell (g(\cdot, x), y)   \in W_{\ouv}^g$. 
  Finally by definition~\eqref{eq:DefTrace} of the trace and by~\eqref{eq:def_W_ell}, $
 [\partial_\x g(\x, x)]^\top \nabla_z \ell (g(\x, x), y) =  \nabla \mathcal{E}_{X, Y} (\x) \in 
  W_{\ouv}^g (\x) \subseteq W_\x^g$ as $\ouv$ is a neighborhood of $\x$.
\end{proof}

\section{Proof of \Cref{lemma:thm} and \Cref{theorem:reformulation_pb}} \label{appendix_loss}
We recall (cf \Cref{ex:param-linear} and \Cref{ex:param-ReLU}) that linear and $2$-layer ReLU neural networks satisfy \Cref{as:main_assumption}, which we recall reads as:\\
{\bf Assumption 2.9}\ (Local reparameterization)\ 
    For each parameter $\theta_0 \in \RD$, for each $x \in \mathcal{X}_{\x_0}$, there is a neighborhood $\ouv$ of $\theta_0$ and a function $f(\cdot, x) \in \mathcal{C}^2(\phi(\ouv), \R^n)$ such that
\begin{equation}
      \forall \x \in \ouv, \quad  
     g(\theta,x) = f(\phi(\theta), x),
\end{equation}
where we also recall that
\begin{equation} \label{eq:X}
\mathcal{X}_{\x_0}\coloneqq \{ x \in \mathcal{X}: \x \mapsto g(\x, x) \text{ is } \mathcal{C}^2 \text{ in the neighborhood of } \x_0\}.
\end{equation}
A common assumption to \Cref{lemma:thm} and \Cref{theorem:reformulation_pb} is that the
loss $\ell(z, y)$ is such that $\ell(\cdot, y)$ is $\mathcal{C}^2$-differentiable for all $y$, 
hence by \Cref{prop:2} and \Cref{prop:3} we have
$$
 W_\x^g
    = \underset{(x, y) \in {\mathcal{X}}_\x \times \mathcal{Y}}{\linspan} \{ [\partial_\x g(\x,x)]^\top \nabla_z \ell (g(\x,x), y) \} \quad \text{ and } \quad  W_\x^g = \partial \phi(\x)^\top  W^{f}_{\phi(\x)}
    $$
where    
    $$
W^{f}_{\phi(\x)} \coloneqq \underset{(x, y) \in {\mathcal{X}}_\x \times \mathcal{Y}}{\linspan} \{ \partial f^x(\phi(\x))^\top \nabla_z \ell (g(\x, x), y) \} \text{ and } f^x(\cdot) \coloneqq f(\cdot, x).
$$
{\bf Consequence of the assumption~\eqref{eq:condition_loss}.}
To proceed further we will rely on the following lemma that shows a direct consequence of \eqref{eq:condition_loss} (in addition to \Cref{as:main_assumption} on the model $g(\x, \cdot)$).
\begin{lemma} \label{lemma:decoupling}
Under 
\Cref{as:main_assumption}, considering a loss   $\ell(z, y)$ such that $\ell(\cdot, y)$ is $\mathcal{C}^2$-differentiable for all $y$. Denote  $f^x(\cdot) \coloneqq f(\cdot, x)$. If the loss satisfies~\eqref{eq:condition_loss}, i.e.
$$\underset{y \in \mathcal{Y}}{\linspan}\{\nabla_z \ell (z, y) \}= \R^n, \forall z \in \R^n,$$
 
    then for all $\theta \in \RD$, 
   \begin{equation}
       W^{f}_{\phi(\x)} =  \underset{(x, w) \in {\mathcal{X}}_\x \times \R^n}{\linspan} \{ \partial f^x(\phi(\x))^\top w \}
   \end{equation} 
\end{lemma}
\begin{proof}
   For $\x \in \RD$, we have
   \begin{align*}
       W^{f}_{\phi(\x)} &=\underset{(x, y) \in {\mathcal{X}}_\x \times \mathcal{Y}}{\linspan} \{ \partial f^x(\phi(\x))^\top \nabla_z \ell (g(\x ,x), y) \} \\
       &= \underset{x \in {\mathcal{X}}_\x}{\linspan} \big\{ \partial f^x(\phi(\x))^\top \underset{y \in \mathcal{Y}}{\linspan} \{ \nabla_z \ell (g(\x , x), y) \} \big\} \\
       &\stackrel{\eqref{eq:condition_loss}}{=} \underset{x \in {\mathcal{X}}_\x}{\linspan} \{ \partial f^x(\phi(\x))^\top \R^n \}\\
       &= \underset{(x, w) \in {\mathcal{X}}_\x \times \R^n}{\linspan} \{ \partial f^x(\phi(\x))^\top w \}.\qedhere
   \end{align*}
\end{proof}
{\bf Verification of~\eqref{eq:condition_loss} for standard ML losses.}
Before proceeding to the proof of \Cref{lemma:thm} and \Cref{theorem:reformulation_pb}, let us show that \eqref{eq:condition_loss} holds for standard ML losses.
\begin{lemma}\label{lem:classicallosses}
The mean-squared error loss $(z, y) \mapsto \ell_2(z, y) \coloneqq \| y-z \|^2$ and the logistic loss $(z \in \R, y \in \{-1, 1 \}) \mapsto \ell_{\mathtt{logis}}(z, y) \coloneqq \log (1+ \exp(-zy) )$ 
satisfy condition \eqref{eq:condition_loss}.
\end{lemma}
\begin{proof}
To show that $\ell_2$ satisfies \eqref{eq:condition_loss} we observe that, with $e_i$ the $i$-th canonical vector, we have
\[
\R^n = \linspan \{e_i: 1 \leq i \leq n\}
= 
\underset{y \in \{z-e_i/2\}_{i=1}^n}{\linspan}\ 2(z-y)
\subseteq 
\underset{y \in \R^n}{\linspan}\  2(z-y) 
=
\underset{y \in \R^n}{\linspan} \nabla_z \ell_2(z, y) 
\subseteq
\R^n.
\]

For the logistic loss, $\nabla_z \ell_{\mathtt{logis}}(z, y) = \frac{-y \exp(-zy)}{1+ \exp(-zy)} \neq 0$ hence $\linspan_y \nabla_z \ell_{\mathtt{logis}}(z, y) =\R$.
\end{proof}
\begin{remark}
In the case of the cross-entropy loss  $(z \in \R^{n}, y \in \{ 1, \cdots, n\}) \mapsto \ell_{\mathtt{cross}}(z,y) \coloneqq - z_y + \log\left(\sum_{i=1}^n \exp z_i \right)$, $\ell_{\mathtt{cross}}$ \textit{does not} satisfy \eqref{eq:condition_loss} as 
$\nabla_z \ell_{\mathtt{cross}} (z, y) = -e_y + \begin{pmatrix}
                \exp(z_1)/(\sum_i \exp z_i)   \\ \cdots \\ \exp(z_n)/(\sum_i \exp z_i)
            \end{pmatrix}
$
satisfies for all $z \in \R^n$: 
$$\linspan_y \nabla_z \ell_{\mathtt{cross}}(z, y) = \{ w\coloneqq(w_1, \cdots, w_n) \in \R^n: \sum w_i = 0 \}=: L_{\mathtt{cross}}.
$$
\end{remark}
An interesting challenge is to investigate variants of \Cref{lemma:thm} under weaker assumptions that would cover the cross-entropy loss.

{\bf The case of linear neural networks of any depth.}
Let us first prove \Cref{lemma:thm} and \Cref{theorem:reformulation_pb}  for the case of linear neural networks.

\begin{theorem}[linear networks]\label{thm:LinearNetsAllF}
Consider a linear network parameterized by $q$ matrices, $\x = (U_1,\ldots, U_q)$ and defined via $g(\x,x) \coloneqq U_1\ldots U_q x$. With $\phi(\x) \coloneqq U_1\ldots U_q \in \R^{n \times m}$ (identified with $\Rd$ with 
$d = nm$), and for any loss $\ell$ satisfying \eqref{eq:condition_loss}, we have for all $\x \in \RD$, $W^{f}_{\phi(\x)}= \Rd$ and $ W_\x^g =\range( \partial \phi(\x)^\top )$. 
\end{theorem}
 \begin{proof}
 Let $\x \in \Theta = \R^D$. As we can factorize the model (cf \Cref{ex:param-linear}) by $g(\cdot, x) = \phi(\cdot) x=: f^x 
(\phi(\cdot)) \in \R^n$ for $x \in \mathcal{X}_{\x} = \R^m$. Thus, by using \Cref{lemma:decoupling}:
 $
W^{f}_{\phi(\x)} 
= \linspan_{x \in \mathcal{X}_{\x}, w \in \R^n} \{  [\partial f^{x}(\phi(\x))]^\top w \}  = \linspan_{x \in \R^m, w \in \R^n} \{  w x^\top \} = \Rd 
 $.
 Finally by \Cref{prop:3}: 
 $
 W_\x^g =
   \partial \phi(\x)^\top  W^{f}_{\phi(\x)} = \range(\partial \phi(\x)^\top).
   $
\end{proof}

{\bf The case of two-layer ReLU networks.}
In the case of two-layer ReLU networks with $r$ neurons, one can write $\x = (U,V,b,c) \in \R^{n \times r} \times \R^{m \times r} \times \R^r \times \R^n$, and denote $u_j$ (resp. $v_j$, $b_j$) the columns of $U$ (resp. columns of $V$, entries of $b$), so that $g_\x(x) = \sum_{j=1}^r u_j \sigma(v_j^\top x+b_j)+c$. 
The set $\mathcal{X}_\x$ (defined in \eqref{eq:X}) is simply the complement in the input domain $\R^m$ of the union of the hyperplanes 
\begin{equation} \label{eq:hyperplan}
    \mathcal{H}_j \coloneqq \{x \in \R^m: v_j^\top x+b_j=0\}.
\end{equation}

\begin{theorem}[two-layer ReLU networks]\label{thm:2layerReLUappendix}
Consider a loss $\ell(z, y)$ satisfying \eqref{eq:condition_loss} and such that $\ell(\cdot, y)$ is $\mathcal{C}^2$-differentiable for all $y$.
On a two-layer ReLU network architecture, let $\x$ be a parameter such that all hyperplanes $\mathcal{H}_j$ defined in \eqref{eq:hyperplan} are pairwise distinct. Then, with $\phi_{\mathtt{ReLU}}$ the reparameterization of \Cref{ex:param-ReLU}, we have:  $W^{f}_{\phi_{\mathtt{ReLU}}(\x)}= \Rd$ and $ W_\x^g =\range( \partial \phi_{\mathtt{ReLU}}(\x)^\top )$. 
\end{theorem}

\begin{proof}
Let $\x$ be a parameter such that all hyperplanes $\mathcal{H}_j$ defined in \eqref{eq:hyperplan} are pairwise distinct. 
Since the loss $\ell$ satisfies \eqref{eq:condition_loss}, by \Cref{lemma:decoupling}, we only need to show that:
$$
    \underset{(x, w) \in {\mathcal{X}}_\x \times \R^n}{\linspan} \{ \partial f^x(\phi(\x))^\top w \} 
= \Rd. 
$$
For convenience, we will use the shorthand $C_{\x,x}$ for the Jacobian matrix $\partial f^x(\phi(\x))$.

\textit{1st case: We consider first the case without bias ($b_j, c = 0$).}
In that case, by \Cref{ex:param-ReLU} we have $\phi(\x) \coloneqq (u_jv_j^\top)_{j=1}^r$ where we write: $\x = (U,V) \in \R^{n \times r} \times \R^{m \times r}$, and denote $u_j$ (resp. $v_j$) the columns of $U$ (resp. columns of $V$).
Here $d=rnm$ and it can be checked (see \Cref{ex:param-ReLU}) that
$$
C_{\x, x}^\top \coloneqq 
            \begin{pmatrix}
                \epsilon_1(x, \theta) A(x)   \\ \cdots \\ \epsilon_r(x, \theta) A(x)
            \end{pmatrix}
            \in \RR^{(rnm) \times n}
$$ 
where:
$$
A: x \in \R^m \mapsto A(x) \coloneqq \begin{pmatrix}
                x & 0 & \cdots & 0 \\ 0 & x & \cdots & 0 \\  \cdots & \cdots &  \cdots &  \cdots \\ 0 & \cdots & 0 & x 
            \end{pmatrix} \in \R^{(nm) \times n},
$$
and where $\epsilon_i(x, \theta) = \mathbb{1} (v_i^\top x > 0)$. 

For $j = 1, \cdots, r$ we denote:
$$
\mathcal{A}_j^+ \coloneqq \{  x \in \R^m: v_j^\top x > 0 \},
\quad\text{and}\quad
\mathcal{A}_j^- \coloneqq \{  x \in \R^m: v_j^\top x < 0 \}.
$$
The open Euclidean ball of radius $r>0$ centered at $c \in \R^m$ is denoted $B(c,r)$.

Consider a hidden neuron $i \in \{1, \cdots, r \}$ and denote $\mathcal{H}_i'  \coloneqq \mathcal{H}_i - \left(\bigcup_{j \neq i}  \mathcal{H}_j \right)$. Since the hyperplanes are pairwise distinct,   $\mathcal{H}_i' \neq \emptyset$ so we can consider an arbitrary $x' \in \mathcal{H}'_i$.
Given any $\eta > 0$, by continuity of $x \in \R^m \mapsto (v_1^\top x, \cdots, v_r^\top x) \in \R^r$, there exists $x_\eta^+ \in B(x', \eta) \cap \mathcal{A}_i^+ $ and  $x_\eta^- \in B(x', \eta) \cap \mathcal{A}_i^-$ such that for all $j \neq i$, $\text{sign}(v_j^\top x_\eta^{\pm}) = \text{sign}(v_j^\top x')$. It follows that $x_\eta^\pm \in \mathcal{X}_{\x}$ (remember that $\mathcal{X}_\x$ is the complement of $\cup_j \mathcal{H}_j$).
As a consequence:
$$
\begin{pmatrix}
                0 \\ \cdots \\  0 \\  A(x')   \\ 0  \\\cdots  \\ 0 
            \end{pmatrix} 
            =
            \lim_{\eta  \to 0}
             \left(C_{\x, x_\eta^+}^\top -  C_{\x, x_\eta^-}^\top   \right)
            \in \overline{\underset{x \in \mathcal{X}_{\x}}{\linspan} \{  C_{\x, x} ^\top  \} } = \underset{x \in \mathcal{X}_{\x}}{\linspan} \{  C_{\x, x} ^\top \},
$$
where the nonzero line in the left-hand-side is the $i$-th, and we used that every finite-dimensional space is closed.

Moreover still by continuity of $x \in \R^m \mapsto (v_1^\top x, \cdots, v_r^\top x) \in \R^r$, there exists $\gamma > 0$, such that for $k = \{-2, -1, 1, 2 \}$, the vectors defined as:
$$
x_{k} \coloneqq x' + \gamma k v_i,
$$
satisfy for all $j \neq i$, $\text{sign}(v_j^\top x_{k}) = \text{sign}(v_j^\top x')$ and $v_i^\top x_k \neq 0$, so that $x_k \in \mathcal{X}_{\x}$ and we similarly obtain
$$
\begin{pmatrix}
                0 \\ \cdots \\  0 \\  \gamma A(v_i)  \\ 0  \\\cdots  \\ 0 
            \end{pmatrix}
            =
C_{\x, x_2}^\top -  C_{\x, x_1}^\top - \left( C_{\x, x_{-1}}^\top -  C_{\x, x_{-2}}^\top \right) \in \underset{x \in \mathcal{X}_{\x}}{\linspan} \{  C_{\x, x} ^\top  \}.
$$
As this holds for every $x' \in \mathcal{H}_i'$,  and since $\linspan\{v_i,\mathcal{H}'_i\} = \R^m$, we deduce that for any $x \in \R^m$
$$
\begin{pmatrix}
                0 \\ \cdots \\  0 \\  A(x)  \\ 0  \\\cdots  \\ 0 
            \end{pmatrix} \in \underset{x \in \mathcal{X}_{\x}}{\linspan} \{  C_{\x, x} ^\top  \}.
$$
As this holds for every hidden neuron $i = 1, \cdots, r$ it follows that for every $x^1, \cdots, x^r \in \R^m$
$$  
\begin{pmatrix}
                A(x^1) \\ \cdots \\  A(x^r)
            \end{pmatrix} \in \underset{x \in \mathcal{X}_{\x}}{\linspan} \{  C_{\x, x} ^\top  \}.
$$
Moreover, by definition of $A(\cdot)$, for each $x \in \R^m$ and each $w = (w_1, \cdots, w_n) \in \R^n$, we have
$$
A(x)w = \begin{pmatrix}
                w_1 x \\ \cdots \\  w_n x
            \end{pmatrix}  \in \R^{nm}.
$$
Identifying $\R^{nm}$ with $\R^{m \times n}$ and the above expression with $x w^\top$, we deduce that
$$
\underset{x \in \R^m, w \in \R^n}{\linspan}  A(x)w = \R^{nm}
$$
and we let the reader check that this implies
$$
\underset{x^1, \cdots, x^r \in \R^m, w \in \R^n}{\linspan}  \begin{pmatrix}
                A(x^1) \\ \cdots \\  A(x^r)
            \end{pmatrix} w
            =
\underset{x^1, \cdots, x^r \in \R^m, w \in \R^n}{\linspan}  \begin{pmatrix}
                A(x^1) w \\ \cdots \\  A(x^r)w
            \end{pmatrix} = \R^{rnm}.
$$
Thus, as claimed, we have
$$
 \underset{x \in \mathcal{X}_{\x}, w \in \R^n}{\linspan} \{  C_{\x, x} ^\top w \}  = \R^{rnm} = \R^d.
$$
\textit{2d case: General case with  
biases.}
 The parameter is $\x = (U,V,b, c) \in \R^{n \times r} \times \R^{m \times r} \times \R^r \times \R^n$ with $b = (b_i)_{i=1}^r$, where $b_i \in \R$ the bias of the $i$-the hidden neuron, and $c$ the output bias.
 
In that case, $d= rn(m+1)$ and one can check that the conditions of \Cref{as:main_assumption} hold with $\phi_{\mathtt{ReLU}}(\x) \coloneqq ((u_iv_i^\top, u_ib_i)_{i=1}^r, c)$ and $f^x(\phi) \coloneqq C_{\x,x}\phi$ where $C_{\x, x}$ is expressed as:
$$
C_{\x, x}^\top \coloneqq 
            \begin{pmatrix}
                \epsilon_1(x, \theta) A'(x)   \\ \cdots   \\ \epsilon_r(x, \theta) A'(x)  \\
                I_n
            \end{pmatrix} \in \R^{(rn(m+1) + n) \times n}
$$ 
where, denoting 
$\bar{x} = (x^\top, 1)^\top \in \R^{m+1}$, we defined

$$
A': x \in \R^m \mapsto A'(x) \coloneqq \begin{pmatrix}
                \bar{x} & 0 & \cdots & 0 \\ 0 & \bar{x} & \cdots & 0 \\  \cdots & \cdots &  \cdots &  \cdots \\ 0 & \cdots & 0 & \bar{x}
            \end{pmatrix} \in \R^{n(m+1) \times n},
$$
and  $\epsilon_i(x, \theta) \coloneqq \mathbb{1} (v_i^\top x + b_i > 0)$. 

Using the sets
$$
\mathcal{A}_j^+ \coloneqq \{  x \in \R^m: v_j^\top x+b_j > 0 \},
\quad\text{and}\quad
\mathcal{A}_j^- \coloneqq \{  x \in \R^m: v_j^\top x+b_j < 0 \},
$$
a reasoning analog to the case without bias allows to show that for each $i = 1, \cdots, r$:
 
$$
\underset{x \in \R^m}{\linspan}  \begin{pmatrix}
                0 \\ \cdots \\  0 \\  A'(x)  \\ 0  \\\cdots  \\ 0 
            \end{pmatrix} \in \underset{x \in \mathcal{X}_{\x}}{\linspan} \{  C_{\x, x} ^\top  \}
$$
so that, again, for every $x^1,\ldots, x^r \in \R^m$ we have
$$
\begin{pmatrix}
                A'(x^1) \\ \cdots \\  A'(x^r) \\ 0
            \end{pmatrix} \in \underset{x \in \mathcal{X}_{\x}}{\linspan} \{  C_{\x, x} ^\top  \}.
$$
As 
$
\begin{pmatrix}
                0 \\ \cdots \\  0 \\ I_n
            \end{pmatrix} \in \underset{x \in \mathcal{X}_{\x}}{\linspan} \{  C_{\x, x} ^\top  \}
$
too, we obtain that
$
\begin{pmatrix}
                A'(x^1) \\ \cdots \\  A'(x^r) \\ I_n
            \end{pmatrix} \in \underset{x \in \mathcal{X}_{\x}}{\linspan} \{  C_{\x, x} ^\top  \}.
$

Now, for each $x \in \R^m$ and $w = (w_1, \cdots, w_n) \in \R^n$, we have
$$
A'(x)w = \begin{pmatrix}
                w_1 \bar{x} \\ \cdots \\  w_n \bar{x}
            \end{pmatrix}  =   \begin{pmatrix}
                w_1 x \\ w_1 \\ \cdots \\  w_n x \\ w_n
            \end{pmatrix}\in \R^{n(m+1)}.
$$
Again, identifying the above expression with $w (x^\top,1) \in \R^{n \times (m+1)}$ it is not difficult to check that
$$
\underset{x \in \R^m, w \in \R^n}{\linspan}  A'(x)w = \R^{n (m+1)},
$$
and we conclude as before.

In both cases we established that $W^f_{\phi_{\mathtt{ReLU}}}(\x) = \Rd$. 
Finally by \Cref{prop:3} we obtain 
 $
 W_\x^g =
   \partial \phi_{\mathtt{ReLU}}(\x)^\top  W^{f}_{\phi_{\mathtt{ReLU}}(\x)} = \range(\partial \phi_{\mathtt{ReLU}}(\x)^\top)
   $.
\end{proof}
Combining \Cref{thm:LinearNetsAllF} and \Cref{thm:2layerReLUappendix} establishes \Cref{lemma:thm} \Cref{theorem:reformulation_pb} as claimed.
One can envision extensions of these results to deeper ReLU networks, using notations and concepts from \cite{stock_embedding_2022} that generalize observations from \Cref{ex:param-ReLU} to deep ReLU networks with biases.  
Given a feedforward network architecture of arbitrary depth, denote $\x$ the collection of all parameters (weights and biases) of a ReLU network on this architecture, and consider $\x \mapsto \phi_{\mathtt{ReLU}}(\x)$ the rescaling-invariant polynomial function of \cite[Definition 6]{stock_embedding_2022} and $C_{\x,x}$, the matrices of \cite[Corollary 3]{stock_embedding_2022} such that the output of the network with parameters $\x$, when fed with an input vector $x \in \R^m$, can be written $g(\x,x) = C_{\x,x} \phi(\x)$. From its definition in \cite[Corollary 3]{stock_embedding_2022}, given $x$, the matrix $C_{\x,x}$  only depends on $\x$ via the so-called {\em activation status} of the neurons in the network (cf \cite[Section 4.1]{stock_embedding_2022}).

\section{Proof of \Cref{lemma-stagnates} } \label{dimstagnates}
\begin{lemma} 
Given $\x \in \Theta$, if for a given $i$, $\vdim (W_{i+1}(\x')) = \vdim( W_i(\x))$ for every $\theta'$ in 
a neighborhood of $\x$, then for all $k \geq i$, we have $W_k(\x')= W_i (\x')$ for all $\x'$ in a neighborhood $\Omega$ of $\x$, where the $W_i$ are defined by \Cref{buildingliealgebra}. Thus $\lie(W)(\x')=  W_i (\x')$ for all $\x' \in \Omega$. In particular, the dimension of the trace of $\lie(W)$ is locally constant and equal to the dimension of $W_i(\x)$.
\end{lemma}
\begin{proof} The result is obvious for $k=i$. The proof is by induction on $k$ starting from $k=i+1$. 
We denote $m:= \dim (W_i(\x))$.

\textit{1st step: Initialization $k=i+1$.}
By definition of the spaces $W_i$ (cf \Cref{buildingliealgebra}) we have $W_i \subset W_{i+1}$ hence $W_i(\x) \subseteq W_{i+1}(\x)$. Since $\vdim  (W_{i+1} (\x)) = \vdim (W_i (\x)) = m$, it follows that there exists $\chi_1, \cdots, \chi_m \in W_i$ such that $\underset{j}{\linspan} \chi_j(\x) = W_i(\x) = W_{i+1}(\x)$ (hence the $m$ vectors $(\chi_1(\x), \cdots, \chi_m(\x))$ are linearly independent). Since each $\chi_j$ is smooth, it follows that $(\chi_1(\x'), \cdots \chi_m(\x'))$ remain linearly independent on some neighborhood $\Omega$ of $\x$, which we assume to be small enough to ensure $\dim W_{i+1}(\x') = m$ for all $\x' \in \Omega$. As $\chi_j \in W_i \subset W_{i+1}$, we obtain that for each $\x' \in \Omega$, the family $\{\chi_j(\x')\}_{j=1}^m$ is a basis of the $m$-dimensional subspace $W_{i+1}(\x')$, hence:
\begin{equation} \label{initialization}
    W_i(\x') \subset W_{i+1} (\x') = \linspan_j {\chi_j(\x')} \subset  W_i(\x'), \quad \forall \x' \in \Omega
\end{equation} 
\textit{2nd step: Induction. } We assume  $W_k(\x') = W_i(\x')$ on $\Omega$. Let us show that $W_{k+1}(\x') = W_i (\x')$ on $\Omega$.
Since $W_{k+1} \coloneqq W_k + [W_0, W_k]$ it is enough to show that $[W_0, W_k](\x') \subseteq W_i(\x')$ on $\Omega$. For this, considering two vector fields, $f \in W_0$ and $\chi \in {W_k}$, we will show that $[f, \chi](\x') \in W_{i+1}(\x')$ for each $\x' \in \Omega$. In light of \eqref{initialization}, this will allow us to conclude. 

Indeed, from the induction hypothesis we know that $ W_{k} (\x') = \linspan_j \chi_j(\x') = W_i(\x')$ on $\Omega$, hence for each $\x' \in \Omega$ there are coefficients $a_j(\x')$ such that $\chi(\x') = \sum_{j=1}^{m} a_j(\x') \chi_j(\x')$. Standard linear algebra shows that these coefficients depend smoothly on $\chi(\x')$ and $\chi_j(\x')$, which are smooth functions of $\x'$, hence the functions $a_j(\cdot)$ are smooth. By linearity of the Lie bracket and of $W_{i+1}(\x')$ it is enough to show that $[f,a_j\chi_j](\x') \in W_{i+1}(\x')$ on $\Omega$ for each $j$. Standard calculus yields

\begin{align*}
[f,a_j\chi_j] &=  
(\partial f) (a_j \chi_j) - \underbrace{\partial (a_j\chi_j)}_{=\chi_j \partial a_j+a_j\partial \chi_j} f
=
a_j [(\partial f) \chi_j-(\partial \chi_j)f]-\chi_j (\partial a_j)f\\
&=
a_j [f,\chi_j] - [(\partial a_j)f] \chi_j
\end{align*}
since $(\partial a_j) f$ is scalar-valued (consider the corresponding dimensions). Since $f \in W_0$ and $\chi_j \in W_i$, by definition of $W_{i+1}$ (cf \Cref{buildingliealgebra}) we have $[f,\chi_j],\chi_j \in W_{i+1}$ hence by linearity we conclude that $[f,a_j\chi_j](\x') \in W_{i+1}(\x')$. As this holds for all $j$, we obtain $[f,\chi](\x') \in W_{i+1}(\x')$. As this is valid for any $f \in W_0$, $\chi \in W_k$ this establishes $[W_0,W_k](\x') \subseteq W_{i+1}(\x')\stackrel{\eqref{initialization}}{=}W_i(\x')$ and we conclude as claimed that $W_i(\x') \subseteq W_{k+1}(\x')= W_k(\x')+[W_0,W_k](\x') \subseteq W_i(\x')$ on $\Omega$.
\end{proof}

\section{Proof of \Cref{mainthm}} \label{appendix_main_theorem}
We recall first the fundamental result of Frobenius using our notations (See Section 1.4 of \cite{isidori}). When we refer to a ``non-singular distribution'', it implies that the dimension of the associated trace remains constant (refer to the definition of ``non-singular'' on page 15 of \cite{isidori}). Being ``involutively consistent'' directly relates to our second assertion using the Lie bracket (see equation 1.13 on page 17 of \cite{isidori}). Lastly, ``completely integrable'' aligns with our first assertion regarding orthogonality conditions (refer to equation 1.16 on page 23 of \cite{isidori}).

\begin{theorem}[Frobenius theorem] \label{frobenius}
    Consider $W 
    \subseteq \vf$, and assume
    that the dimension of $W(\x)$ is constant on $\ouv \subseteq \RD$. 
    Then the two following assertions are equivalent:
\begin{enumerate}[leftmargin=0.5cm]
    \item each $\x \in \ouv$ admits a neighborhood $\Omega'$ such that there exists $D-\vdim (W(\x))$ independent conserved functions through $W_{|\Omega'}$;
    \item the following property holds:
    \eq{
        \quad \quad [u,v](\x) 
        \in W(\x),\quad
        \text{for each}\  u, v \in W,\ \x \in \ouv
        \label{eq:frob-crochets-app}
    }
\end{enumerate}
\end{theorem}

\begin{proposition} \label{remarkFro}
Under the assumption that $\vdim (W(\x))$ is locally constant on $\Omega$, Condition~\eqref{eq:frob-crochets-app} of Frobenius Theorem holds if, and only if, the linear space $W' \coloneqq \{ \chi \in \vf, \forall \x \in \ouv: \chi(\x) \in W(\x) \}$ (which is a priori infinite-dimensional) is a Lie algebra.
\end{proposition}
\begin{proof}
$\Leftarrow$ If $W'$ is a Lie algebra, then as $W \subset W'$ we get: for all  $u, v \in W \subset W', [u, v] \in W'$. Given the definition of $W'$ this means that \eqref{eq:frob-crochets-app} is satisfied. 

$\Rightarrow$ Assuming now that \eqref{eq:frob-crochets-app} holds, we prove that $W'$ is a Lie algebra.
For this, given $X, Y \in W'$ we wish to show that $[X,Y](\x) \in W(\x)$ for every $\x \in \Omega$.

Given $\x \in \Omega$, we first reason as in the first step of the proof of \Cref{lemma-stagnates} to obtain the existence of a neighborhood $\Omega'$ of $\x$ and of $m \coloneqq \vdim (W(\x'))$ vector fields
$\chi_1, \cdots, \chi_m \in W$ such that  
$(\chi_1(\x'), \cdots, \chi_m(\x'))$ is a basis of $W(\x')$ for each $\x' \in \Omega$. 
By definition of $W'$ we have $X(\x') \in W(\x')$ and $Y(\x') \in W(\x')$ for every $\x' \in \Omega'$. Thus, there are smooth functions $a_j,b_j$ such that $X(\cdot) = \sum_{1}^{m} a_i(\cdot) \chi_i(\cdot)$ and $Y(\cdot) = \sum_{1}^{m} b_i(\cdot) \chi_i(\cdot)$ on $\Omega'$, and we deduce by bilinearity of the Lie brackets that 
$[X, Y](\x') =  \sum_{i, j} [a_i \chi_i, b_j \chi_j](\x')$ on $\Omega'$. Since $W(\x)$ is a linear space, we will conclude that $[X,Y](\x) \in W(\x)$ if we can show that $[a_i \chi_i, b_j \chi_j](\x) \in W(\x)$. Indeed, we can compute
\begin{align*}
     [a_i \chi_i, b_j \chi_j]
     &= a_ib_j [\chi_i, \chi_j]
     +
     b_j [(\partial a_i) \chi_j] \chi_i - a_i [(\partial b_j) \chi_i] \chi_j
\end{align*}
   where, due to dimensions, both $(\partial a_i) \chi_j$ and $(\partial b_j) \chi_i$ are smooth scalar-valued functions. 
   By construction of the basis $\{\chi_j\}_j$ we have $\chi_i(\x), \chi_j(\x) \in W(\x)$, and by assumption~\eqref{eq:frob-crochets-app} we have $[\chi_i, \chi_j](\x) \in W(\x)$, hence we conclude that $[X, Y](\x) \in W(\x)$. Since this holds for any choice of $X,Y \in W'$, this establishes that $W'$ is a Lie algebra.
  \end{proof}

\begin{theorem}
If $\vdim (\lie(W_\phi)(\x))$ is locally constant then each $\x \in \ouv$ has a neighborhood $\Omega'$ such that there are $D-\vdim (\lie(W_\phi)(\x))$ (and no more) independent conserved functions through ${W_\phi}_{|\Omega'}$.
\end{theorem}

\begin{proof}
\textit{1st step: Existence of $\Omega'$ and of $D-\vdim (\lie(W_\phi)(\x))$ independent conserved functions.}
    Let $\x \in \ouv$. Since $\vdim (\lie(W_\phi)(\x))$ is locally constant there is a neighborhood $\ouv''$ of $\x$ on which it is constant. Since $W \coloneqq\lie(W_\phi)_{\mid \ouv''} \subseteq \mathcal{X}(\Omega'')$ is a Lie Algebra, by \Cref{remarkFro} and Frobenius theorem (\Cref{frobenius}) 
    there exists a neighborhood $\Omega' \subseteq \ouv''$ of $\x$ and $D-\vdim (W(\x))$ independent conserved functions through $W_{\mid \Omega'}$. As $W_\phi \subset \lie(W_\phi)$, these functions are (locally) conserved through $W_\phi$ too. We only need to show that there are no more conserved functions.

\textit{2nd step: There are no more conserved functions.}
    By contradiction, assume there exists $\x_0 \in \ouv$, an open neighborhood $\Omega'$ of $\x_0$, a dimension $k < \vdim (\lie(W_\phi)(\x_0))$, and a collection of $D-k$
    independent conserved functions through $W_\phi$, gathered as the coordinates of a vector-valued function
    $h \in \mathcal{C}^1(\Omega', \mathbb{R}^{D-k})$.
        Consider $W \coloneqq \{ X \in \mathcal{X}(\Omega'), \forall \x \in \Omega', X(\x) \in \mathrm{ker} \partial h(\x) \}$. 
        By the definition of independent conserved functions, the rows of the $(D-k) \times D$ Jacobian matrix
    $\partial h(\x)$ are linearly independent
    on $\Omega'$,  and the dimension of $W(\x) = \mathrm{ker}\partial h(\x)$ is constant and equal to $k$ on $\Omega'$. 
       By construction of $W$ and \Cref{tracecharact}, the $D-k$ coordinate functions of $h$ are independent conserved functions through $W$. Thus, by Frobenius Theorem (\Cref{frobenius}) and Proposition \ref{remarkFro}, $W$ is a Lie algebra. 
    By \Cref{tracecharact} we have $W_\phi (\x) = \range \partial \phi(\x)^\top\subset \mathrm{ker} \partial h(\x)$ on $\Omega'$, hence ${W_\phi}_{|\Omega'} \subset W$, and therefore $\lie(W_\phi)_{|\Omega'} = \lie({W_\phi}_{|\Omega'}) \subset W$. 
    In particular:
    $\lie(W_\phi)(\x_0) \subset W(\x_0)$, which leads to the claimed contradiction that $\vdim (\lie(W_\phi)(\x_0))  \leq \vdim (W(\x_0)) = k$.
\end{proof}

\section{Proofs of the Examples of \Cref{subsection-method} and additional example} \label{annex:ExamplesSection3}
\subsection{Proof of the result given in \Cref{shallowNN}} \label{appA1}

\begin{proposition}\label{prop:shallowNNnobias}
Consider $\x = (U,V) \in \R^{n\times r} \times \R^{m \times r}$, $\phi$, and $\Omega \subseteq \RD$, $D = (n+m)r$, as in \Cref{shallowNN}. 
The dimension of $W_\phi(\x)$ is constant and equal to $(n+m-1)r$ and $W_\phi$ verifies condition~\eqref{eq:frob-crochets} of Frobenius Theorem (i.e. condition~\eqref{eq:frob-crochets-app} of~\Cref{frobenius}). 
\end{proposition}

\begin{proof}
Denoting $u_i$ (resp. $v_i$) the columns of $U$ (resp. of $V$), for $\x \in \Omega$ we can write $\phi(\x) = (\psi(u_i, v_i))_{i = 1, \cdots r}$ with $\psi: (u \in \R^n - \{0\}, v \in  \R^m - \{ 0 \}) \mapsto uv^\top 
\in \R^{n\times m}$. 
As this decouples $\phi$ into $r$ functions each depending on a separate block of coordinates, Jacobian matrices and Hessian matrices are block-diagonal. Establishing condition~\eqref{eq:frob-crochets-app} of Frobenius theorem is thus equivalent to showing it for each block, which can be done by dealing with the case $r=1$. Similarly, $W_\phi(\x)$ is a direct sum of the spaces associated to each block, hence it is enough to treat the case $r=1$ (by proving that the dimension is $n+m-1$) to obtain that for any $r \geq 1$ the dimension is $r(n+m-1)$.

\textit{1st step: We show that $W_\phi$ satisfies condition~\eqref{eq:frob-crochets-app} of Frobenius Theorem.}
For $u \in \R^n - \{0\}$, $v \in \R^m - \{0\}$
we write $\x = (u;v) \in \RD = \R^{n+m}$ and $\phi_{i, j} (\x) \coloneqq u_i v_j$ for $i = 1, \cdots, n$ and $j = 1, \cdots, m$. Now $u_i$ and $v_j$ are \textit{scalars} (and no longer columns of $U$ and $V$).
Denoting $e_i\in \RD = \R^{n +m}$ the vector such that all its coordinates are null except the $i$-th one, we have:
    \begin{align*}
\nabla \phi_{i, j}(\x) &= v_j e_i + u_i e_{n + j} \in \RD,\\
\partial ^{2} \phi_{i, j}(\x) &= E_{j+n, i} + E_{i, j + n} \in \R^{D \times D},
\end{align*}
with $E_{i, j} \in \R^{ D \times D}$ 
the one-hot matrix with the $(i, j)$-th entry being $1$. 
Let $i, k \in \{1, \cdots, n \}$ and $j, l \in \{1, \cdots, m \}$. 

\textit{1st case: $(i, j) = (k, l)$}
Then trivially $
\partial ^{2} \phi_{i, j} (\x) \nabla \phi_{k, l} (\x) - \partial ^{2} \phi_{k, l} (\x) \nabla \phi_{i, j}(\x)  = 0.
$\\
\textit{2nd case: $( (i \neq  k)$ and $(j \neq l))$}
Then 
$$
[\nabla \phi_{i, j}, \nabla \phi_{k, l}](\x)  =  (E_{j+n, i} + E_{i, j + n}) (v_l e_k + u_k e_{n + l}) - (E_{l+n, k} + E_{k, l + n}) (v_j e_i + u_i e_{n + j}) = 0 - 0.
$$
\textit{3d case: $i = k$ and $j \neq l$.}
Then as $u \neq 0$, there exists $l' \in \{1, \cdots, n \}$ such that $u_{l'} \neq 0$.
\begin{align*}
\partial ^{2} \phi_{i, j} (\x) \nabla \phi_{k, l} (\x) - \partial ^{2} \phi_{k, l} (\x) \nabla \phi_{i, j}(\x) 
&=  v_l e_{n + j} - v_j e_{n +l}  \\
&= \frac{v_l}{u_{l'}} \nabla \phi_{l', j} (\x) - \frac{v_j}{u_{l'}} \nabla \phi_{l', l} (\x),  \\
&\in \text{span} \{ \nabla \phi_{i, j} (\x) \} = W_\phi(\x).
\end{align*}
\textit{4d case: ($(i \neq k)$ and $(j = l)$)}
We treat this case in the exact same way than the 3d case.

Thus $W_\phi$ verifies condition~\eqref{eq:frob-crochets} of Frobenius Theorem.

\textit{2d step: We show that $\vdim (W_\phi (\x) )= (n + m -1)$.}
As $u, v \neq 0$ each of these vectors has at least one nonzero entry. For simplicity of notation, and without loss of generality, we assume that $u_1 \neq 0$ and $v_1 \neq 0$.
It is straightforward to check that
 $(\nabla \phi_{1, 1} (\x), (\nabla \phi_{1, j}(\x))_{j =2, \cdots, m}, (\nabla \phi_{i, 1}(\x))_{i =2, \cdots, n })$ are $n+m-1$ linearly independent vectors. 
To show that $\vdim (W_\phi (\x)) = (n + m -1)$ is it thus sufficient to show that they span $W_\phi(\x)$. This is a direct consequence of the fact that, for any $i, j$, we have
\begin{align*}
   \nabla \phi_{i, j} (\x)  = v_j e_i + u_i e_{n + j}  
   &= \frac{v_j}{v_1} (v_1 e_i + u_i e_{n +1}) + \frac{u_i}{u_1}(u_1 e_{n + j} + v_j e_1) -  \frac{v_j u_i}{u_1 v_1} \left(u_1 e_{n +1} + v_1 e_1 \right), \\
   &= \frac{v_j}{v_1} \nabla \phi_{i, 1}(\x) + \frac{u_i}{u_1} \nabla \phi_{1, j}(\x) + \frac{v_j u_i}{u_1 v_1} \nabla \phi_{1, 1}(\x). \qedhere
\end{align*}
\end{proof}

\subsection{An additional example beyond ReLU} \label{appA2}

In complement to \Cref{shallowNN}, we give a simple example studying a two-layer network with a positively homogeneous activation function, which include the ReLU but also variants such as the leaky ReLU or linear networks.
\begin{example}[Beyond ReLU: Neural network with one hidden neuron with a positively homogeneous activation function of degree one] \label{exNN} 
Let $\sigma$ be a positively one-homogeneous activation function. 
In~\eqref{eq:elr-general}, this corresponds to setting $g(\x,x) = \sum_{i=1}^r u_i \sigma( \langle v_i, x\rangle ) \in \mathbb{R}$.  
Assuming $\langle v_i, x\rangle \neq 0$ for all $i$ to avoid the issue of potential non-differentiability at $0$ of $\sigma$ (for instance for the ReLU), and in particular assuming $v_i \neq 0$, the function minimized during training can be factored via $\phi(\x) = (\psi(u_i,v_i))_{i=1}^r$ where 
\begin{equation} \label{paramNN}
\x \coloneqq (u \in \R, v \in \R^{d-1}-\{0 \}) \overset{\psi}{\mapsto} (u \|v\|, v/ \|v\|) \in \R \times \mathcal{S}_{d-1} \subset \R^{d}.
\end{equation}

\begin{proposition}\label{prop:homogeneousNN}
Consider $d \geq 2$ and $\phi(\x) = (\psi(u_i,v_i))_{i=1}^r$ where $\psi$ is given by \eqref{paramNN} on $\ouv \coloneqq \{\x = (u \in \R^r,V = (v_1,\ldots,v_r) \in \R^{m \times r}): v_i \neq 0\}$. We have $\vdim (W_\phi(\x)) = r(d-1)$ and $W_\phi$ verifies condition~\eqref{eq:frob-crochets-app} of Frobenius Theorem (\Cref{frobenius}),  so each $\x=(u, V) \in \ouv$ admits a neighborhood $\Omega'$ such that there exists $r$ (and no more) conserved function through ${W_\phi}_{|\Omega'}$. 
\end{proposition}
As in \Cref{shallowNN}, such candidate functions are given by $h_i: (u_i, v_i) \mapsto u_i^2 - \|v_i\|^2$. A posteriori, these functions are in fact conserved through all $W_\phi$.
\end{example}
\begin{proof}[Proof of \Cref{prop:homogeneousNN}]
As in the proof of~\Cref{prop:shallowNNnobias} it is enough to prove the result for $r = 1$ hidden neuron.
   Note that here $D = d$. To simplify notations, we define $\phi_0, ..., \phi_{d-1}$ for $\x = (u,v)$ as:
$$
\phi_0(\x) = u \| v \|,
$$
and for $i = 1, ..., d-1$:
$$
\phi_i(\x) = v_i/ \| v \|.
$$

\textit{1st step: explicitation of $\text{span} \{ \nabla \phi_0, ..., \nabla \phi_{d-1}\} $}. 
We have
\[
\partial \phi(\x) =
\begin{pmatrix}
  \|v \| 
  & \rvline &   uv^\top / \|v \|  \quad \\
\hline
  0_{(d-1) \times 1} & \rvline &
  \begin{matrix}
   &  \\
   \frac{1}{\|v\|} P_v\\
   & 
  \end{matrix}
\end{pmatrix},
\]
where:
$
P_v := \mathrm{I}_{d-1} - v v^\top / \| v \|^2
$ is the orthogonal projector on $(\R v)^{\perp}$ (seen here as a subset of $\R^{d-1}$) and its rank is $d-2$. Thus $\vdim (W_\phi(\x)) = \mathtt{rank}(\partial \phi(\x))=d-1$ and $\text{span} \{ \nabla \phi_0, ..., \nabla \phi_{d-1}\} = \R \nabla \phi_0 + (\R v)^{\perp}$.

\textit{2d step: calculation of the Hessians.}

\textit{1st case: The Hessian of $\phi_i$ for $i \geq 1$.} 
In this case, $\phi_i$ does not depend on the first coordinate $u$ so we proceed as if the ambient space here was $\R^{d-1}$.
We have already that for $i \geq 1$:
$$
\nabla \phi_i(\x) = e_i / \| v \| - v_i v / {\| v \|}^3
$$
hence
$$
\partial ^{2} \phi_i = 3 v_i v v^\top / {\| v \|}^5 - 1 / {\| v \|}^3 \left( v_i \mathrm{I}_{d-1} + V_i + V_i^\top \right),
$$
where all columns of matrix $V_i := (0, ..., v, 0, ...,0)$ are zero except the $i$-th one, which is set to $v$.

\textit{2d case: The Hessian of $\phi_0$}.
Since
$$
\nabla \phi_0 (\x) = {\left(\| v \|, u v^\top/ \| v \| \right)}^\top.
$$
we have 
\[
\partial^2 \phi_0(\x) =
\begin{pmatrix}
  0
  & \rvline &   v^\top / \|v \|  \quad \\
\hline
  v / \|v \| & \rvline &
  \begin{matrix}
   &  \\
   \frac{u}{\|v\|} P_v \\
   & 
  \end{matrix}
\end{pmatrix}.
\]

\textit{3rd step: Conclusion.} 

\textit{1st case: $i, j \geq 1$ and $i \neq j$.} We have:
\begin{align*}
&\partial ^{2} \phi_i(\x) \nabla \phi_j (\x)- \partial ^{2} \phi_j(\x) \nabla \phi_i(\x), \\
&=v_j /  {\|v\|}^4 e_i -   v_i/  {\|v\|}^4 e_j \in (\R v)^{\perp},\\
&\subset \text{span} \{ \nabla \phi_0(\x), ..., \nabla \phi_{d-1}(\x)\}.
\end{align*}

\textit{2d case: $i \geq 1$ and $j=0$.} 
We have:
\begin{align*}
&\partial ^{2} \phi_i (\x)\nabla \phi_0 (\x)- \partial ^{2} \phi_0 (\x)\nabla \phi_i (\x), \\
&= -2 u  /  {\|v\|} \nabla \phi_i (\x), \\
&\in \text{span} \{ \nabla \phi_0(\x), ..., \nabla \phi_{d-1} (\x)\}.\qedhere
\end{align*}
In both cases, we obtain as claimed that the condition \eqref{eq:frob-crochets-app} of Frobenius Theorem is satisfied, and we conclude using the latter.
\end{proof}

\section{Proof of \Cref{lowerdimflow} and additional example} \label{appendix-lowerdim}
\begin{proposition}
Assume that $\mathtt{rank}(\partial \phi(\x))$ is constant  
on $\Omega$ 
and that $W_\phi$ satisfies \eqref{eq:frob-crochets}. 
If $t \mapsto \x(t)$ satisfies the ODE \eqref{gradientflow} then there is $0< T^\star_{\theta_{\mathtt{init}}}<T_{\theta_{\mathtt{init}}}$ such that $z(t) \coloneqq \phi(\x(t)) \in \Rd$ 
satisfies the ODE 
 \begin{equation}  \label{transfertGF-app}
\left\{
    \begin{array}{ll}
        \overset{.}{z}(t) &=  -M(z(t), \x_{\text{init}}) \nabla f(z(t)) \quad \mbox{for all } 
        0 \leq t < T^\star_{\theta_{\mathtt{init}}},
        \\
        z(0) &= \phi(\x_{\text{init}}),
    \end{array}
\right.
\end{equation}
where $M(z(t),\x_{\text{init}}) \in \R^{d \times d }$ is a symmetric positive semi-definite matrix. 
\end{proposition}
\begin{proof}
    As  $z = \phi(\x)$ and as $\x$ satisfies \eqref{gradientflow}, we have:
\begin{align*}
\overset{.}{z} = \partial \phi (\x)  \overset{.}{\x}
 &= - \partial \phi (\x)  \nabla (f\circ \phi)(\x)
= - \partial \phi (\x) [\partial \phi(\x)]^\top \nabla f (z).
\end{align*}
Thus, we only need to show 
$M(t) \coloneqq \partial \phi (\x(t)) [\partial \phi(\x(t))]^\top$, which is a symmetric, positive semi-definite $d \times d$ matrix, 
only depends on $z(t)$ and $\x_{\text{init}}$.
Since $\dim W_\phi(\x)=\mathtt{rank}(\partial \phi(\x))$ is constant on $\Omega$ and $W_\phi$ satisfies~\eqref{eq:frob-crochets}, by
Frobenius Theorem (\Cref{frobenius}), 
for each $\x \in \ouv$, there exists a neighborhood $\Omega_1$ of $\x$ and $D-d'$ independent conserved functions $h_{d'+1}, \cdots, h_{D}$ through $(W_\phi)_{|\Omega'}$, with $d' := \dim W_\phi(\x) = \mathtt{rank}(\partial\phi(\x))$.  
Moreover, by definition of the rank, for the considered $\x$, there exists a set $I \subset \{1,\ldots,d\}$ of $d'$ indices such that the gradient vectors $\nabla \phi_i(\x)$, $i \in I$ are linearly independent. By continuity, they stay linearly independent on a neighborhood $\Omega_2$ of $\x$. Let us denote $P_I$ the restriction to the selected indices and
$$
\x' \in \RD \longmapsto \Phi_I(\x') \coloneqq (
P_I\phi(\x'),
h_{d'+1}(\x'), ..., h_{D}(\x')) \in \RD
$$
As the functions $h_i$ are \textit{independent} conserved functions, for each $\x' \in \Omega' \coloneqq \Omega_1 \cap \Omega_2$ their gradients $\nabla h_i(\x')$, $d'+1 \leq i \leq D$ are both linearly independent and (by \Cref{tracecharact} and~\eqref{eq:v-phi}) orthogonal to $W_\phi(\x') = \range [\partial \phi(\x')]^\top = \linspan \{\nabla \phi_i(\x): i \in I\}$.
Hence, on $\Omega'$, the Jacobian $\partial \Phi_I$ is an invertible $D \times D$ matrix. By the implicit function theorem, the function $\Phi_I$ is thus locally invertible. 
Applying this analysis to $\x = \x(0)$ and using that $h_i$ are conserved functions, we obtain that in an interval $[0,T^\star_{\theta_{\mathtt{init}}})$ 
we have
\begin{equation} \label{dep}
\Phi_I(\x(t)) = (P_I z(t), h_{d+1}(\x_{\text{init}}), ..., h_{D}(\x_{\text{init}}))
\end{equation}
By local inversion of $\Phi_I$ this allows to express $\x(t)$ (and therefore also $M(t) = \partial \phi(\x(t)) [\partial \phi(\x(t))]^\top$) as a function of $z(t)$ and of the initialization. 
\end{proof}

In complement to \Cref{ex:riemannian1} we provide another example related to \Cref{exNN}.
\begin{example}\label{ex:riemannian2}
Given the reparametrization 
$\phi: (u \in \R, v \in \R^{d-1}-\{0 \}) \mapsto (u \|v\| , v/ \|v\|) \in \R \times \mathcal{S}_{d-1} \subset \Rd$ (cf \eqref{paramNN}), the variable $z \coloneqq (r, h) = (u \|v\| , v/ \|v\|) $ satisfies \eqref{transfertGF-app} with:
    $M(z, \x_{\text{init}})~=~\begin{pmatrix}
  \sqrt{r^2 + \delta^2}
  & \rvline &   \quad 0_{1 \times k}  \quad \\
\hline
  0_{(d-1) \times 1} & \rvline &
  \begin{matrix}
   &  \\
   \frac{1}{\delta + \sqrt{r^2 + \delta^2}} P_h \\
   & 
  \end{matrix}
\end{pmatrix},
$
where $P_h := \mathrm{I}_{d-1} - h h^T / \| h \|^2$ and $\delta \coloneqq u_{\text{init}}^2-\| v _{\text{init}} \|^2$.
\end{example}

\section{Proofs of results of \Cref{sectionlinearNN}} \label{appendix:2layerLNN}
\subsection{Proof of Proposition \ref{classicinv}} \label{appC}
\begin{proposition} 
     Consider $\Psi: (U, V) \mapsto U^\top U - V^\top V \in \R^{r \times r}$ and assume that $(U; V)$ has full rank. Then:
     \begin{enumerate}
        \itemsep0em 
         \item if $n+m \leq r$, the function $\Psi$ gives $(n+m)(r-1/2(n+m-1))$ independent conserved functions,
         \item if $n+m > r$, the function $\Psi$ gives $r(r+1)/2$ independent conserved functions.
     \end{enumerate}
\end{proposition}
\begin{proof}
    Let write $U = (U_1; \cdots; U_r)$ and $V= (V_1; \cdots;V_r)$ then:
    $\Psi_{i, j}(U, V) = \langle U_i, U_j \rangle - \langle V_i, V_j \rangle$ for $i, j = 1, \cdots, r$. Then $f_{i, j} \coloneqq \nabla \Psi_{i, j} (U, V) = (0; \cdots; 0; \underset{(i)}{U_j}; \cdots; \underset{(j)}{U_i}; 0; \cdots; \underset{(i+r)}{-V_j}; \cdots; \underset{(j+r)}{-V_i}; \cdots; 0)^\top \in \R^{(n+m)r \times 1}$.

    \textit{1st case: $n+m \leq r$.} As $(U; V)$ has full rank, its rank is $n+m$. Thus, by elementary operations, $(U;-V)$ has full rank $(n+m)$ too.
    Let us note $(U; -V) = (w_1, \cdots, w_r)$.
    Without loss of generality we can assume that its $(n+m)$ first columns $w_1, \cdots, w_{(n+m)}$ are linearly independent.
    Now we want to count the number of $f_{i,j}$ that are linearly independent.
\begin{enumerate}
    \item if $i\in [\![ 1, n+m ]\!]$ and if $i \leq j \leq r$, then let us show that all the associated $f_{i, j}$ are linearly independent together. There are $(n+m)(n+m+1)/2 + (n+m)(r-(n+m))$ such functions. 
    Indeed, let us assume that there exists $(\lambda_{i, j})_{i, j}$ such that $\sum_{1 \leq i \leq j \leq n+m} \lambda_{i, j} f_{i, j} = 0$.
    Let $j_0 \in [\![ 1, r ]\!]$. By looking this sum at the restriction on the coordinates $[\![ n j_0, n (j_0 +1 ) ]\!] \cup  [\![ nr + m j_0, nr + m (j_0 +1) ]\!]$, we obtain:
    $$
    \sum_{i \leq \max(j_0, n+m)} \lambda_{i, j_0} w_i = 0.
    $$
    Then by independence of $w_1, \cdots, w_{n+m}$, $\lambda_{i, j_0} = 0$ for all $i \leq \max(j_0, n+m)$, and for all $j_0 \leq r$. Thus all these $f_{i, j}$ are linearly independent.
    \item if $i\leq j \in [\![ n+m +1, r]\!]$, the associated $f_{i, j}$ are linearly dependent of thus already obtained. Indeed as $w_i$ and $w_j$ are linear combinations of $\{w_1, \cdots w_{n+m}\}$, there exists $(\alpha_j)_j \neq (0)$ and $(\beta_j)_j \neq (0)$ such that $w_i = \sum_{k =1}^{n+m} \alpha_k w_k$ and $w_j = \sum_{k=1}^{n+m} \beta_k w_k$. Then, one has:
    $$
    f_{i, j} = \sum_k \beta_k f_{k, i} + \sum_k \alpha_k f_{k, j} - \sum_{k, l} \alpha_k \beta_l f_{k, l}.
    $$
    \end{enumerate}
Finally there are exactly $(n+m)(r-1/2(n+m-1)$ independent conserved functions given by $\Psi$.

\textit{2d case: $n+m > r$.} Then all $(U_i; -V_i)$ for $i = 1, \cdots r$ are linearly independent. Then there are $r(r+1)/2$ independent conserved functions given by $\Psi$.
\end{proof}

\subsection{Proofs of other results}

\begin{proposition} \label{linear_vf}
For every $\Delta \in \R^{n \times m}$ denote
$
S_{\Delta} \coloneqq  \left(\begin{matrix}
        0 & \Delta\\
        \Delta^\top & 0
    \end{matrix}\right),
$
one has
$
    \partial \phi(U, V)^\top : \Delta \in \R^{n \times m} \mapsto  
    S_\Delta \cdot
        (U; V). 
$
Hence 
$W_\phi = \linspan \{A_{\Delta}, \forall \Delta \in \R^{n \times m} \}$, where
$
    A_{\Delta}: (U; V) \mapsto  
    S_\Delta
    \cdot (U; V)
    $
is a linear endomorphism. 
Moreover one has $[A_{\Delta}, A_{\Delta'}]: (U, V) \mapsto  [S_{\Delta}, S_{\Delta'}]\times(U; V).$
\end{proposition}

This proposition enables the computation of the Lie brackets of $W_\phi$ by computing the Lie bracket of matrices. In particular, $\lie(W_\phi)$ is necessarily of finite dimension.

\begin{proposition} \label{lieW}
    The Lie algebra $\lie(W_\phi)$ is equal to 
    $$
        \left\{ 
            (U; V) \mapsto 
            \begin{pmatrix}
                \mathrm{I}_{n} & 0 \\ 0 & -\mathrm{I}_{m}
            \end{pmatrix}
            %
            \times M \times             
            \begin{pmatrix}
                U\\V
            \end{pmatrix}
            : M \in \mathcal{A}_{n+m} 
        \right\}
    $$
    where $\mathcal{A}_{n+m} \subset \R^{(n+m)\times (n+m)}$ is the space of skew symmetric matrices.
\end{proposition}


\begin{remark}
    By the characterization of $\lie(W_\phi)$ in Proposition \ref{lieW} we have that the dimension of $\lie(W_\phi)$ is equal to $(n+m)\times (n+m-1)/2.$
\end{remark}

\begin{proof}
\textit{1st step: Let us characterize $W_1 = \linspan \{W_\phi + [W_\phi, W_\phi] \}$}.
Let $\Delta, \Delta' \in \R^{n \times m}$, then: 
\begin{equation} \label{liebracketmatrix}
[A_{\Delta}, A_{\Delta'}]((U, V)) = [S_{\Delta}, S_{\Delta'}] \times(U; V) =   \begin{pmatrix}
                Y& 0 \\ 0& Z
            \end{pmatrix}
            \times             
            \begin{pmatrix}
                U\\V
            \end{pmatrix},
\end{equation}
with $Y \coloneqq \Delta \Delta'^\top - \Delta' \Delta^\top \in \mathcal{A}_n$ and $Z \coloneqq \Delta^\top \Delta'- \Delta'^\top \Delta \in \mathcal{A}_m$.
Then: 
\begin{align*}
W_1 &= \left\{ (U; V) \mapsto  \begin{pmatrix}
                Y& X \\ X^\top& Z
            \end{pmatrix} \times  \begin{pmatrix}
                U\\V
            \end{pmatrix}: X \in \R^{n \times m }, Y \in \mathcal{A}_{n},  Z \in \mathcal{A}_m \right\}, \\
&= \left\{ u_M \coloneqq (U; V) \mapsto  \begin{pmatrix}
                \mathrm{I}_{n}& 0 \\ 0& -\mathrm{I}_{m}
            \end{pmatrix}
            \times M \times             
            \begin{pmatrix}
                U\\V
            \end{pmatrix}:  M \in \mathcal{A}_{n+m} \right\}. 
\end{align*}
\textit{2d step: Let us show that $W_2 = W_1$.}
Let  $M, M' \in \mathcal{A}_{n+m}$. Then:
$$
    [u_M, u_{M'}] =  \begin{pmatrix}
                \mathrm{I}_{n}& 0 \\ 0& -\mathrm{I}_{m}
            \end{pmatrix} \left(M \begin{pmatrix}
                \mathrm{I}_{n}&0 \\ 0& -\mathrm{I}_{m}
            \end{pmatrix} M' -  M' \begin{pmatrix}
                \mathrm{I}_{n}&0 \\ 0& -\mathrm{I}_{m}
            \end{pmatrix} M\right) 
    = \begin{pmatrix}
                \mathrm{I}_{n}& 0 \\ 0& -\mathrm{I}_{m}
            \end{pmatrix} \tilde{M},
$$
with $\tilde{M} \coloneqq M \begin{pmatrix}
                \mathrm{I}_{n}& 0 \\ 0& -\mathrm{I}_{m}
            \end{pmatrix} M' -  M' \begin{pmatrix}
                \mathrm{I}_{n}& 0 \\ 0& -\mathrm{I}_{m}
            \end{pmatrix} M \in \mathcal{A}_{n+m}$.

Finally: $\lie(W_\phi) = W_1 = \left\{ (U; V) \mapsto \begin{pmatrix}
                \mathrm{I}_{n}& 0 \\ 0&  -\mathrm{I}_{m}
            \end{pmatrix} \times M   \times             
            \begin{pmatrix}
                U\\V
            \end{pmatrix}:  M \in \mathcal{A}_{n+m} \right\}. $
\end{proof}

Eventually, what we need to compute is the dimension of the trace $\lie(W_\phi) (U, V)$ for any $(U, V)$. 

\begin{proposition}  \label{dim-lie-algebra}
Let us assume that $(U; V) \in \R^{(n+m) \times r}$ has full rank. Then:
\begin{enumerate}
    \item if $n+m \leq r$, then $\vdim (\lie(W_\phi) (U; V)) = (n+m)(n+m-1)/2$;
    \item if $n+m > r$, then $\vdim (\lie(W_\phi) (U; V)) = (n+m)r - r(r+1)/2$.
\end{enumerate}
\end{proposition}
\begin{proof}
Let us consider the linear application:
$$
\Gamma: M \in \mathcal{A}_{n+m}\mapsto \begin{pmatrix}
                \mathrm{I}_{n}& 0 \\ 0& -\mathrm{I}_{m}
            \end{pmatrix}
            %
            \times M \times             
            \begin{pmatrix}
                U\\V
            \end{pmatrix},
$$
where $\mathcal{A}_{n+m} \subset \R^{(n+m)^2}$ is the space of skew symmetric matrices. As $\range 
\Gamma (\mathcal{A}_{n+m}) = \lie(W_\phi)(U;V)$, we only want to calculate $\textrm{rank} \Gamma ( \mathcal{A}_{n+m})$. But by rank–nullity theorem, we have:
$$
\vdim\ \operatorname{ker}\ \Gamma + \operatorname{rank}\ \Gamma = (n+m)(n+m-1)/2.
$$
\textit{1st case: $n+m \leq r$.} Then as $\begin{pmatrix}
                \mathrm{I}_{n}& 0 \\ 0& -\mathrm{I}_{m}
            \end{pmatrix}$ is invertible and as $(U; V)$ has full rank $n+m$, $\Gamma$ is injective and then $\textrm{rank} \Gamma ( \mathcal{A}_{n+m})=  (n+m)(n+m-1)/2$.

\textit{2d case:  $n+m > r$.} Since $\begin{pmatrix}
                \mathrm{I}_{n}& 0 \\ 0& -\mathrm{I}_{m}
            \end{pmatrix}$ is invertible,
            $\operatorname{ker}( \Gamma)$ is the set of matrices $M \in \mathcal{A}_{n+m}$ such that $M (U;V)=0$. 
            Denote $M_i$, $1 \leq i \leq 2(n+m)$ the rows of such a matrix, so that $M^\top = (M_1; \cdots; M_{2(n+m)})$. Denoting
 $C_j$, $1 \leq j \leq r$ the columns of $\left(U; V\right)$
and $\mathcal{C} \coloneqq \underset{j = 1, \cdots, r}{\linspan} C_j$, we observe that since  $\left(U; V\right)$ has full rank $r = \min(m+n,r)$ the columns $C_j$ are linearly independent and $\vdim\ \mathcal{C} = r$.
Since $M\times \left(U; V\right) = 0$, we have $\langle M_i, C_j \rangle = 0$ for all $1 \leq i \leq n+m$ and $1 \leq j \leq  r$, i.e., each $M_i \in \R^{n+m}$ belongs to $\mathcal{C}^\perp$, of dimension $\vdim\ \mathcal{C}^\perp = n+m-r$. 

To determine  $\vdim\ \operatorname{ker}(\Gamma)$ we now count the number of degrees of freedom to choose $M \in \mathcal{A}_{n+m}$ such that $M_i \in \mathcal{C}^\perp$ for every $i$. 
We first show the following lemma.
\begin{lemma} \label{lemma:completion}
    The matrix $C \in \R^{(n+m) \times r}$ has full rank $r$ if and only if there exists a subset $T$ of $n+m-r$ indices such that the horizontal concatenation $\left(C, \mId{T}\right)$ is invertible, where $\mId{T} \in \R^{(n+m) \times (n+m-r)}$ is {the restriction of the identity matrix to its columns indexed by }$T$.
\end{lemma}

\begin{proof}
    The converse implication is clear. Let us show the direct one.
    By denoting $e_1, \cdots, e_{n+m}$ the canonical basis in $\R^{n+m}$, there is $i_1$ such that $e_{i_1}$ is linearly independent from all $C_j$: otherwise all $e_i$ would be spanned by $C_1, \cdots, C_r$, i.e. we would have ${\linspan} \{ e_i: 1 \leq i \leq n+m\}\subseteq \mathcal{C}$ hence  $n+m \leq r$, which contradicts our assumption. 
    Similarly, by recursion, after finding $i_1, \cdots, i_{k}$ for some $k < n+m-r$ such that $i_1, \cdots, i_k$ are linearly independent from $C_1, \dots, C_r$ (so that $\tilde{\mathcal{C}} \coloneqq {\linspan} \{\mathcal{C} , e_{i_l}: 1\leq l \leq k \}$ has dimension $r+k < n+m$), there exists $i_{k+1}$ such that $e_{i_{k+1}}$ is linearly independent from all $C_j$ and all $e_{i_1}, \cdots, e_{i_k}$.  
    Stopping this construction when $k = n+m-r$ yields $T:= \{i_1,\ldots, i_k\}$.
\end{proof}
Consider the index set $T = \{i_1, \cdots, i_{n+m-r} \}$ given by~\Cref{lemma:completion}, so that
$\left(C, \mId{T}\right) \in \R^{(n+m) \times (n+m)}$ is invertible.

We first build the column $M_{i_1}$. The $i_1$-th coordinate of $M_{i_1}$ is equal to $0$ as $M$ is a skew matrix, and its remaining $n+m-1$ coordinates can be freely chosen provided that
$M_{i_1}$ belongs to  $\mathcal{C}^\perp$.
 Thus, $M_{i_1}$ can be arbitrarily chosen in the space of dimension $n+m-r-1$ defined by 
$$
\left(C, e_{i_1}\right)^\top M_{i_1} = \begin{pmatrix}
    0 \\ \cdots \\ 0 \\ 0
\end{pmatrix},
$$
where the matrix $\left(C, e_{i_1}\right)^\top \in \R^{(r+1) \times (n+m)}$ has full rank $r+1$ by contruction.

Then, the $i_2$-th coordinate of $M_{i_2}$ is equal to $0$, and its $i_1$-th coordinate is determined by $M_{i_1}$ (and equal to the opposite of its $i_2$-th one) as $M$ is a skew matrix. Its remaining $n+m-2$ coordinates can be freely chosen provided that
$M_{i_2}$ belongs to  $\mathcal{C}^\perp$. Thus, $M_{i_2}$ can be arbitrarily chosen in the affine space of dimension $n+m-r-2$ defined by 
$$
\left(C, e_{i_1}, e_{i_2}\right)^\top M_{i_2} = \begin{pmatrix}
    0 \\ \cdots \\ 0 \\- M_{i_1}[i_2] \\ 0
\end{pmatrix},
$$
where the matrix $\left(C, e_{i_1}, e_{i_2} \right)^\top \in \R^{(r+2) \times (n+m)}$ has full rank $r+2$ by construction.
By recursion, after building $k$ columns $M_{i_1}, \cdots, M_{i_k}$ with $k < n+m-r$, the coordinates indexed by $i_1, \cdots, i_k$ of the column $M_{i_{k+1}}$ are determined by $M_{i_1}, \cdots, M_{i_k}$  and its $i_{k+1}$-th coordinate is equal to zero to ensure that $M$ is a skew matrix, and the remaining $n+m-(k+1)$ coordinates must ensure that   $M_{i_{k+1}} \in \mathcal{C}^\perp$. Thus $M_{i_{k+1}}$ can be arbitrarily chosen in the affine space of dimension 
$n+m-r-(k+1)$ defined by 
$$
\left(C, e_{i_1}, \cdots, e_{i_{k+1}} \right)^\top M_{i_{k+1}} = \begin{pmatrix}
    0 \\ \cdots \\ 0 \\ -M_{i_1}[i_{k+1}] \\ \cdots \\ -M_{i_k} [i_{k+1}] \\0
\end{pmatrix},
$$
where the matrix $\left(C, e_{i_1}, \cdots, e_{i_{k+1}} \right)^\top\in \R^{(r+k+1) \times (n+m)}$ has full rank $r+k+1$ by construction.
Finally the dimension of $\mathrm{ker}(\Gamma)$ is equal to:
$$
\sum_{i=1}^{n+m-r}(n+m-r-i) = (n+m-r-1)(n+m-r)/2.
$$
Eventually we obtain $\operatorname{rank}(\Gamma)= (n+m)r -r(r+1)/2$.
\end{proof}

Thanks to this explicit characterization of the trace of the generated Lie algebra, combined with \Cref{classicinv}, we conclude that \Cref{conservation}  has indeed exhausted the list of independent conservation laws. 

\begin{corollary} 
    If $(U; V)$ has full rank, then all conserved functions are given by $\Psi: (U, V) \mapsto U^\top U - V^\top V$. In particular, there exist no more conserved functions.
\end{corollary}
\begin{proof}
 As $(U;V)$ has full rank, this remains locally the case. By~\Cref{dimlie} the dimension of $\lie(W_\phi) (U;V)$ is locally constant, denoted $m(U,V)$. By Theorem \ref{mainthm}, the exact number of independent conserved functions is equal to $(n+m)r - m(U, V)$ and that number corresponds to the one given in Proposition~\ref{classicinv}.
\end{proof}

\section{About \Cref{ex-contre-ex-frobenius}} \label{contre-ex-frobenius}
\begin{proposition}
Let us assume that $(U; V) \in \R^{(n+m) \times r}$ has full rank. If $\max(n, m) > 1$ and $r > 1$, then $W_\phi$ does not satisfy the condition  \eqref{eq:frob-crochets}.
\end{proposition}
\begin{proof}
Let us consider the linear application:
$$
\Gamma': \Delta \in \R^{n \times m} \mapsto \begin{pmatrix}
                0 & \Delta \\ \Delta^\top &  0
            \end{pmatrix}
            \times 
            \begin{pmatrix}
                U\\V
            \end{pmatrix}.
            $$
    By \Cref{linear_vf},  $\range \Gamma' (\R^{n \times m}) = W_\phi(U;V)$. Thus, as by definition $W_\phi(U;V) \subseteq \lie(W_\phi (U;V))$, $W_\phi$ does not satisfy the condition \eqref{eq:frob-crochets} if and only if $\vdim (W_\phi(U;V)) < \vdim (\lie W_\phi (U;V))$.

    \textit{1st case: $n+m \leq r$.} Then as $(U; V)$ has full rank $n+m$, $\Gamma'$ is injective and then $\textrm{rank} \Gamma' ( \R^{n \times m})=  n \times m$.

    Thus by \Cref{dim-lie-algebra}, we only need to verify that:
    $ n\times m < (n+m)(n+m-1)/2 =: \lie \W (U;V)$. It is the case as $\max(n, m) > 1$.

\textit{2d case:  $n+m > r$.} We write $(U;V) = (C_1; \cdots; C_r)$ with $(C_1, \cdots, C_r)$ that are linearly independent as $(U; V)$ has full rank $r$. 
Let $\Delta \in  \R^{n \times m}$ such that $\Gamma'(\Delta) = 0$. Let us define the symmetric matrix $M$ by:
\begin{equation} \label{form-m}
M \coloneqq \begin{pmatrix}
                0& \Delta \\ \Delta^\top&  0
            \end{pmatrix}.
\end{equation}
Then $M\cdot (U;V) = 0$. Then we write $M^\top = (M_1; \cdots; M_{n+m})$. Then as  $M\times (U;V) = 0$, we have that $\langle M_i, C_j \rangle = 0$ for all $i = 1, \cdots, n+m$ and for all $j = 1, \cdots, r$. We note $C \coloneqq \underset{i = 1, \cdots, r}{\linspan} C_i$ that is of dimension $r$ as $(U;V)$ has full rank $r$.

For all $i = 1, \cdots, n$, $M_i$ must be in $C^\perp$ and its $n$ first coordinate must be zero by definition \eqref{form-m}. Then $M_i$ lies in a space of dimension $\max(0, n+m-r-n)$.
For all $j > n$, $M_j$ are entirely determined by $\{ M_i \}_{i \leq n}$ by definition \eqref{form-m}.
Finally the dimension of $\mathrm{ker} \Gamma'$ is equal to:
$
n \times \max(0, m-r).
$
Then: $\vdim (\W (U; V)) = \textrm{rank} \Gamma' ( \R^{n \times m})= nm - n \times \max(0, m-r)$.

Thus by \Cref{dim-lie-algebra}, we only need to verify that:
    $ nm - n\max(0, m-r) < (n+m)r - r(r+1)/2 =: \lie \W (U;V)$. 

\textit{Let us assume $m < r$.} Then by looking at $f(r) \coloneqq  (n+m)r - r(r+1)/2 - nm = \vdim (\lie \W (U;V)) -  \vdim (\W (U;V)) $ for $r \in \{m+1, \cdots,  n + m -1 \}=: I_{n, m}$, we have: $f'(r) = (n+m) - 1/2 - r > 0$ (as $n+m> r$ is an integer), so $f$ is increasing, so on $I_{n, m}$, we have (as $r> m$): $f(r) > f(m) = (n+m)m - m(m+1)/2-nm  = m^2 - m(m+1)/2 \geq 0$ as $m \geq 1$.

\textit{Let us assume $m \geq r$.} Then
\begin{align*}
    \vdim (\lie \W (U;V)) -  \vdim (\W (U;V)) &= (n+m)r - r(r+1)/2 - (nm -n(m-r)), \\
    &= mr - r(r+1)/2, \\
    &\geq r^2 - r(r+1)/2 \quad \text{as } m\geq r, \\
    &>  0 \quad \text{as } r > 1.
\end{align*}
Thus $\vdim (\lie \W (U;V)) -  \vdim (\W (U;V) )> 0$.
\end{proof}

\section{Details about experiments} \label{appendix_numeric}
We used the software SageMath \cite{sagemath} that relies on a Python interface. Computations were run in parallel using 64 cores on an academic HPC platform.

First we compared the dimension of the generated Lie algebra $\lie(\W)(\x)$ (computed using the algorithm presented in \Cref{subsection-method}) with $D-N$, where $N$ is the number of independent conserved functions known by the literature (predicted by \Cref{conservation} for ReLU and linear neural networks). 
We tested both linear and ReLU architectures (with and without biases) of various depths and widths, and observed that the two numbers matched in all our examples.

For this, we draw $50$ random ReLU (resp. linear) neural network architectures, with depth drawn uniformly at random between $2$ to $5$ and i.i.d. layer widths drawn uniformly at random between $2$ to $10$ (resp. between $2$ to $6$). For ReLU architectures, the probability to include biases was $1/2$.

Then we checked that all conservation laws can be explicitly computed using the algorithm presented in \Cref{section:constructibility} and looking for polynomial solutions of degree $2$ (as conservation laws already known by the literature are polynomials of degree $2$). As expected we found back all known conservation laws by choosing $10$ random ReLU (resp. linear) neural network architectures with depth drawn uniformly at random between $2$ to $4$ and i.i.d. layer widths drawn uniformly at random between $2$ to $5$. 

\end{document}